\documentclass{article}
\usepackage[nonatbib,final]{neurips_2023}
\usepackage[utf8]{inputenc} 
\usepackage[T1]{fontenc}    
\usepackage{hyperref}       
\usepackage{url}            
\usepackage{booktabs}       
\usepackage{amsfonts}       
\usepackage{nicefrac}       
\usepackage{microtype}      
\usepackage{xcolor}         
\hypersetup{
    colorlinks,
    linkcolor={blue!50!black},
    citecolor={blue!50!black},
    urlcolor={blue!80!black}
}
\usepackage[toc,page]{appendix}
\usepackage{wrapfig}
\usepackage{amsmath}
\usepackage{amssymb}
\usepackage{mathtools}
\usepackage{amsthm}
\usepackage{bbm}
\usepackage{braket}
\usepackage[capitalize,noabbrev]{cleveref}
\usepackage{algorithm} 
\usepackage{algorithmic}
\usepackage{afterpage}
\usepackage[sort]{natbib}
\usepackage{minitoc}
\setcitestyle{numbers,open={[},close={]}}
\theoremstyle{plain}
\newtheorem{theorem}{Theorem}[section]
\newtheorem{proposition}[theorem]{Proposition}
\newtheorem{lemma}[theorem]{Lemma}
\newtheorem{corollary}[theorem]{Corollary}
\theoremstyle{definition}

\newtheorem{assumption}{Assumption}
\theoremstyle{remark}
\newtheorem{remark}[theorem]{Remark}

\newcommand{\mmd}{\operatorname{MMD}}

\newcommand{\N}{\mathbb{N}}

\newcommand{\R}{\mathbb{R}}

\DeclareMathOperator*{\argmin}{arg\,min}

\DeclareMathOperator{\KL}{KL}
\DeclareMathOperator{\TV}{TV}

\DeclareMathOperator\var{var}

\newcommand\mc[1]{\mathcal{#1}}

\newcommand{\D}{\mathrm{d}}

\DeclarePairedDelimiter{\norm}{\lVert}{\rVert}

\DeclareMathOperator\E{\mathbb{E}}

\DeclareMathOperator\one{\mathbbm{1}}

\newcommand{\eps}{\varepsilon}

\DeclareMathOperator*{\argmax}{arg\,max}
\newcommand{\bb}{\mathbb}

\newcommand{\ellp}{{\ell'}}

\def\eqdef{\triangleq}
\renewcommand{\sf}{\mathsf}
\renewcommand{\var}{\operatorname{var}}
\newcommand{\cov}{\operatorname{cov}}
\newcommand{\coef}{\operatorname{Coef}}
\renewcommand{\one}{\mathbbm{1}}
\renewcommand{\cal}{\mathcal}
\renewcommand{\P}{\bb{P}}

\newcommand{\LF}{\sf{LFHT}}
\newcommand{\mLF}{\sf{mLFHT}}

\newcommand{\sref}[2]{\hyperref[#1]{#2 \ref*{#1}}}
\newcommand\numberthis{\addtocounter{equation}{1}\tag{\theequation}}
\newcommand{\Sch}{\textup{Scheff\'e}}

\def\eqdef{\triangleq}
\renewcommand{\sf}{\mathsf}
\renewcommand{\var}{\operatorname{var}}
\renewcommand{\one}{\mathbbm{1}}
\renewcommand{\cal}{\mathcal}
\renewcommand{\P}{\bb{P}}

\def\eqdef{\triangleq}
\renewcommand{\sf}{\mathsf}
\renewcommand{\var}{\operatorname{var}}
\renewcommand{\one}{\mathbbm{1}}
\renewcommand{\cal}{\mathcal}
\renewcommand{\P}{\bb{P}}

\renewcommand{\eps}{\epsilon}

\makeatletter
\newenvironment{breakablealgorithm}
  {
   \begin{center}
     \refstepcounter{algorithm}
     \hrule height.8pt depth0pt \kern2pt
     \renewcommand{\caption}[2][\relax]{
       {\raggedright\textbf{\ALG@name~\thealgorithm} ##2\par}%
       \ifx\relax##1\relax 
         \addcontentsline{loa}{algorithm}{\protect\numberline{\thealgorithm}##2}%
       \else 
         \addcontentsline{loa}{algorithm}{\protect\numberline{\thealgorithm}##1}%
       \fi
       \kern2pt\hrule\kern2pt
     }
  }{
     \kern2pt\hrule\relax
   \end{center}
  }
\makeatother

\usepackage{xr}
\makeatletter
\newcommand*{\addFileDependency}[1]{
  \typeout{(#1)}
  \@addtofilelist{#1}
  \IfFileExists{#1}{}{\typeout{No file #1.}}
}
\makeatother

\newcommand*{\myexternaldocument}[1]{
    \externaldocument{#1}
    \addFileDependency{#1.tex}
    \addFileDependency{#1.aux}
}

\myexternaldocument{supplement}

\title{Kernel-Based Tests for \\ Likelihood-Free Hypothesis Testing}

\author{%
  Patrik R\'obert Gerber\thanks{Equal contribution.} \\
  Department of Mathematics, MIT \\
  Cambridge, MA 02139 \\
  \texttt{prgerber@mit.edu} 
  \\
  \And Tianze Jiang\footnotemark[1]  \\ Department of Mathematics, MIT \\
  Cambridge, MA 02139 \\
  \texttt{tjiang@mit.edu} 
  \\
  \And Yury Polyanskiy\footnotemark[1]  \\ Department of EECS, MIT \\ Cambridge, MA 02139 \\ \texttt{yp@mit.edu} 
  \\ 
  \And Rui Sun\footnotemark[1]  \\ Department of Mathematics, MIT \\ Cambridge, MA 02139 \\ \texttt{eruisun@mit.edu}
}

\begin{document}
\doparttoc
\part{}
\maketitle

\begin{abstract}
Given $n$ observations from two balanced classes, consider the task of labeling an additional $m$ inputs that are known to all belong to \emph{one} of the two classes. 
Special cases of this problem are well-known: with complete
knowledge of class distributions ($n=\infty$) the
problem is solved optimally by the likelihood-ratio test; when
$m=1$ it corresponds to binary classification; and when $m\approx n$ it is equivalent to two-sample testing. The intermediate settings occur in the field of likelihood-free inference, where labeled samples are obtained by running forward simulations and the unlabeled sample is collected experimentally. In recent work it was discovered that there is a fundamental trade-off
between $m$ and $n$: increasing the data sample $m$ reduces the amount $n$ of training/simulation
data needed. In this work we (a) introduce a generalization where unlabeled samples 
come from a mixture of the two classes -- a case often encountered in practice; (b) study the minimax sample complexity for non-parametric classes of densities under \textit{maximum mean
discrepancy} (MMD) separation; and (c) investigate the empirical performance of kernels parameterized by neural networks on two tasks: detection
of the Higgs boson and detection of planted DDPM generated images amidst
CIFAR-10 images. For both problems we confirm the existence of the theoretically predicted asymmetric $m$ vs $n$ trade-off.

\end{abstract}

\section{Likelihood-Free Inference}


The goal of likelihood-free inference (LFI) \cite{diggle1984monte,gutmann2018likelihood,brehmer2020mining,cranmer2020frontier}, also
called simulation-based inference (SBI), is to perform statistical inference in a setting where the
data generating process is a black-box, but can be simulated. Given the ability to generate
samples $X_\theta \sim P_\theta^{\otimes n}$ for any parameter $\theta$, and given real-world data
$Z \sim P_{\theta^\star}^{\otimes m}$, we want to use our simulations to learn about the truth
$\theta^\star$. LFI is particularly relevant in areas of science where we have precise but complex
laws of nature, for which we can do (stochastic) forward simulations, but
can not directly compute the (distribution) density $P_\theta$. The Bayesian community 
approached the problem under the name of Approximate Bayesian Computation (ABC) \cite{csillery2010approximate,sisson2018handbook,beaumont2019approximate}. More recent ML-based methods where regressors and
classifiers are used to summarize data, select regions of interest, approximate likelihoods or
likelihood-ratios \cite{izbicki2014high,jiang2017learning,papamakarios2019sequential,dalmasso2020confidence,thomas2022likelihood} have also emerged for this challenge. 

Despite empirical advances, the theoretical study of frequentist LFI is still in its infancy. We focus on the nonparametric and non-asymptotic setting, which we justify as follows. For
applications where tight error control is critical one might be reluctant to rely on asymptotics.
More broadly, the non-asymptotic regime can uncover new phenomena and provide insights for algorithm design. Further, parametric models are clearly at odds with the black-box assumption. Recently,
\cite{RY22SBI} proposed likelihood-free hypothesis testing \eqref{eqn:LFHT def} as a simplified model
and found minimax optimal tests for a range of nonparametric distribution classes, thereby identifying   a
\textit{fundamental simulation-experimentation trade-off} between the number of simulated
observations $n$ and the size of the experimental data sample $m$. Here we extend~\cite{RY22SBI}, and prior related work \cite{huang2012classification,huang2012hypothesis,kelly2012classification,kelly2010universal,gerber2023minimax}, to a new setting designed to model experimental setups more truthfully and derive sample complexity (upper and lower bounds) for kernel-based tests over nonparametric classes.

While minimax optimal, the algorithms of \cite{RY22SBI,gerber2023minimax} are impractical as they rely on discretizing
the observations on a regular grid. Thus, both in our theory as well as experiments we turn to kernel methods which provide an empirically more powerful set of
algorithms that have shown success in nonparametric testing
\cite{gretton2009fast,gretton2012kernel,jitkrittum2018informative,gretton2012optimal,
liu2020learning}. 

\textbf{Contributions}\quad  Our contributions are twofold. \textit{Theoretically}, we 
introduce \emph{mixed likelihood-free hypothesis testing} \eqref{eqn:lfht mix def}, which is a generalization of \eqref{eqn:LFHT def} and provides a better model of applications such as the search for new
physics \cite{cowan2011asymptotic,lista2017statistical}. We propose a robust kernel-based test and derive both upper and lower bounds on its minimax sample complexity over a large nonparametric class of densities, generalizing multiple results in \cite{kelly2010universal,huang2012classification,huang2012hypothesis,kelly2012classification,li2019optimality,RY22SBI,gerber2023minimax}. Although the simulation-experimentation ($m$ vs $n$) trade-off has been proven in the minimax sense (that is, for some worst-case data distribution), it is not clear whether it actually occurs in real data. Our second contribution is the \textit{empirical} confirmation of the existence of an asymmetric trade-off, cf. Figure \ref{fig:trade off}. To this end we construct state-of-the-art tests building on ideas of \cite{sutherland2016generative,liu2020learning} on learning good kernels from the data. We execute this program in two settings: the Higgs boson discovery \cite{baldi2014searching}, and detecting diffusion \cite{ho2020denoising} generated images planted in the CIFAR-10 \cite{krizhevsky2009learning} dataset.

\subsection{$\LF$ and the Simulation-Experimentation Trade-off}\label{sec:lfht}
Suppose we have i.i.d. samples $X,Y$ each of size $n$ from two unknown distributions $P_X,P_Y$ on a measurable space $\cal X$, as well as a third i.i.d. sample $Z \sim P_Z$ of size $m$. In the context of LFI, we may think of the samples $X,Y$ as being generated by our simulator, and $Z$ being the data collected in the real world. The problem we refer to as likelihood-free hypothesis testing is the task of deciding between the hypotheses
\begin{equation}\label{eqn:LFHT def}\tag{$\LF$}
    H_0: P_Z=P_X \quad\text{versus}\quad H_1: P_Z = P_Y. 
\end{equation}
This problem originates in \cite{gutman1989asymptotically,ziv1988classification},
where authors study the exponents of error decay for finite $\cal X$ and fixed $P_X,P_Y$ as $n
\sim m \to \infty$; more recently \cite{kelly2010universal,huang2012classification,huang2012hypothesis,kelly2012classification,acharya2012competitive,li2019optimality,RY22SBI,gerber2023minimax} it is studied in the non-asymptotic regime. Assuming that $P_X,P_Y$ belong to a known nonparametric class of distributions $\cal P$ and are
guaranteed to be $\epsilon$-separated with respect to total variation (TV) distance (i.e. $\TV(P_X,P_Y)\geq\epsilon$), \cite{RY22SBI} characterizes the sample sizes $n$ and $m$ required for the sum
of type-I and type-II errors to be small, as
a function of $\epsilon$ and for several different $\cal P$'s. Their results show, for three settings of $\cal P$, that $(i)$ testing \eqref{eqn:LFHT def} at vanishing error is possible even when $n$
is not large enough to estimate $P_X$ and $P_Y$ within total variation distance $\cal O(\epsilon)$, and that $(ii)$ to achieve a fixed level of error, say $\alpha$, one can
\textit{trade off}
$m$ vs. $n$ along a curve of the form $\{\min\{n, \sqrt{mn}\} \gtrsim n_\sf{TS}(\alpha, \epsilon, \cal P), m\gtrsim \log(1/\alpha)/\epsilon^2\}$. Here $n_\sf{TS}$ denotes the minimax sample complexity of two-sample testing over $\cal P$, i.e. the minimum number of observations $n$ needed from $P_X,P_Y \in \cal P$ to distinguish  the cases $\TV(P_X,P_Y)\geq\epsilon$ versus $P_X=P_Y$. Here $\gtrsim$ suppresses dependence on constants and untracked parameters. 

It is unclear, however, whether predictions drawn from minimax sample complexities over specified distribution classes can be observed in real-world data. Without the theory, a natural expectation is that the error contour $\{(m,n):\exists\text{ a test with total error $\leq\alpha$}\}$ would look similar to that of minimax two-sample testing with unequal sample size, namely $\{(m,n):\min\{n, m\}\gtrsim n_\sf{TS}(\alpha, \epsilon, \cal P)\}$, i.e. $n$ and $m$ simply need to be above a certain threshold simultaneously (as is the case for e.g. two-sample testing over smooth densities \cite{arias2018remember,li2019optimality}). However, from Figures \ref{fig:toy trade off} and \ref{fig:trade off} we see that there is indeed a non-trivial trade-off between $n$ and $m$: the contours are not always parallel to the axes and aren't symmetric about the line $m=n$. The importance of Fig.~\ref{fig:trade off} is in demonstrating that said trade-off is not a kink of a theory that arises due to some esoteric worst-case data distribution, but is instead a real effect observed in state-of-the-art LFI algorithms ran on actual data. We remark that the $n$ used in this plot is the total number of simulated samples (most of which are used for choosing a neural-network parameterized kernel) and are not just the $n$ occuring in \Cref{thm: MMDUpper,thm:lower bds} which apply to a \textit{fixed} kernel. See \Cref{sec:learning kernels} for details on sample division.

\subsection{Mixed Likelihood-Free Hypothesis Testing}
\label{sec:new physics}
A prominent application of likelihood-free inference lies in the field of particle physics. Scientists run sophisticated experiments in the hope of finding a new particle or phenomenon. Often said phenomenon can be predicted from theory, and thus can be simulated, as was the case for the Higgs boson whose existence was verified after nearly $50$ years at the Large Hadron Collider (LHC) \cite{chatrchyan2012observation,adam2015higgs}. 

Suppose we have $n$ simulations from the \emph{background} distribution $P_X$ and the \emph{signal} distribution $P_Y$. Further, we also have $m$ (real-world) datapoints from $P_Z=(1-\nu)P_X+\nu P_Y$, i.e. the observed data is a mixture between the background and signal distributions with rate parameter $\nu$. The goal of physicists is to construct confidence intervals for $\nu$, and a \emph{discovery} corresponds to a $5\sigma$ confidence interval that excludes $\nu=0$. We model this problem by testing
\begin{equation}\label{eqn:lfht mix def}\tag{$\sf{mLFHT}$}
    H_0:\nu=0 \quad\text{versus}\quad H_1:\nu \geq \delta
\end{equation}
for fixed (usually predicted) $\delta>0$. See the rigorous definition of \eqref{eqn:lfht mix def} in \Cref{sec:rates}. In particular, a discovery can be claimed if $H_0$ is rejected.

\begin{wrapfigure}{R}{0.4\textwidth}
    \vspace{-20mm}
    \begin{center}
    \includegraphics[width=0.45\textwidth]{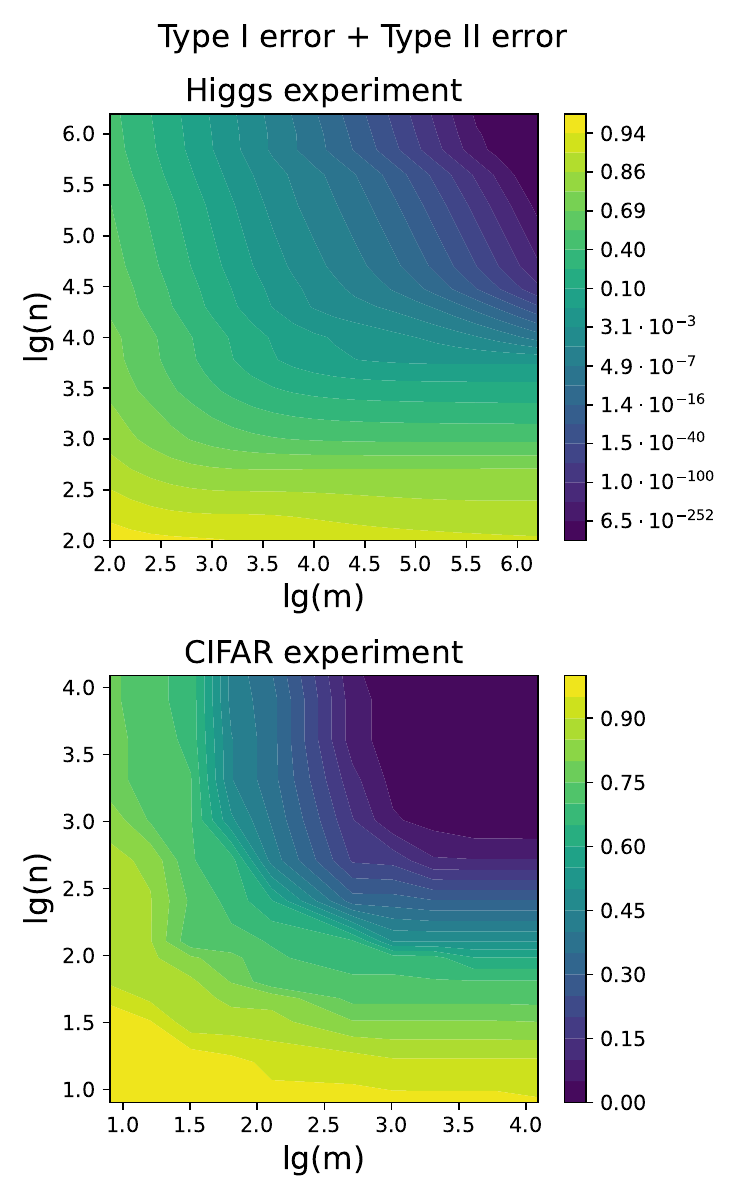}
    \vspace{-5mm}
    \caption{\small
    $n$ versus $m$ trade-off for the Higgs and CIFAR experiments using our test in \Cref{sec:test_stat}. Error probabilities are estimated by normal approximation for Higgs and simulated for CIFAR.}
    \end{center}
    \label{fig:trade off}
    \vspace{-10mm}
\end{wrapfigure}

\section{The Likelihood-Free Test Statistic}\label{sec:test_stat}

This section introduces the testing procedure based on Maximum Mean Discrepancy (MMD) that we study throughout the paper both theoretically and empirically. First, we introduce the necessary background on MMD in Section \ref{sec: MMDintro}. Then, we define our test statistics in Section \ref{sec:the statistic}.

\subsection{Kernel Embeddings and MMD}\label{sec: MMDintro}
Given a set $\cal X$, we call the function $K:\cal X^2 \to \R$ a kernel if the $n\times n$ matrix with $ij$'th entry $K(x_i,x_j)$ is symmetric positive semidefinite for all choices of $x_1,\dots,x_n \in \cal X$ and $n\geq1$. There is a unique reproducing kernel Hilbert space (RKHS) $\cal H_K$ associated to $K$. $\cal H_K$ consists of functions $\cal X \mapsto \R$ and satisfies the reproducing property $\langle K(x,\cdot), f\rangle_{\cal H_K} = f(x)$ for all $f \in \cal H_k$ and $x \in \cal X$, in particular $K(x,\cdot) \in \cal H_K$. Given a probability measure $P$ on $\cal X$, define its kernel embedding $\theta_P$ as
\begin{equation}\label{eqn:def embedding}
    \theta_P \eqdef \E_{X\sim P} K(X,\cdot) = \int_{\cal X} K(x,\cdot) P(\D x). 
\end{equation}
Given the kernel embeddings of two probability measures $P,Q$, we can measure their distance in the RKHS by $\mmd(P,Q) \eqdef \|\theta_P-\theta_Q\|_{\cal H_K}$, where MMD stands for maximum mean discrepancy. MMD has a closed form thanks to the reproducing property and linearity:
\begin{align*}
    \mmd^2(P,Q) &= \E\Big[K(X,X')+K(Y,Y')-2K(X,Y)\Big] 
\end{align*}
where $(X,X',Y,Y')\sim P^{\otimes 2}\otimes Q^{\otimes 2}$. In particular, if $P,Q$ are empirical measures based on observations, we can evaluate the MMD exactly, which is crucial in practice. Yet another attractive property of MMD is that (under mild integrability conditions) it is an integral probability metric (IPM) where the supremum is over the unit ball of the RKHS $\cal H_K$. See e.g. \cite{10.7551/mitpress/4175.001.0001,muandet2017kernel} for references. The following result is a consequence of the fact that self-adjoint compact operators are diagonalizable.

\begin{theorem}[Hilbert–Schmidt]\label{thm:HS}
Suppose that $K \in L^2(\mu\otimes\mu)$ is symmetric. Then there exists a sequence $(\lambda_j)_{j\geq 1} \in \ell^2$ and an orthonormal basis $\{e_j\}_{j\geq1}$ of $L^2(\mu)$ such that $ K(x,y) = \sum_{j\geq1}\lambda_j e_j(x)e_j(y) $ for all $j\geq1$, 
    where convergence is in $L^2(\mu\otimes\mu)$. 
\end{theorem}

\begin{assumption}
    Unless specified otherwise, we implicitly assume a choice of a non-negative measure $\mu$ and kernel $K$ for which the conditions of \Cref{thm:HS} hold. Note that $\lambda_j\geq 0$ and depend on $\mu$. 
\end{assumption}

\textbf{Removing the Bias}\label{sec:diagonal}\quad
In our proofs we work with the kernel embedding of empirical measures for which we need to modify the inner product $\langle\cdot,\cdot\rangle_{\cal H_K}$ (and thus $\mmd$) slightly by removing the diagonal terms. Namely, given i.i.d. samples $X,Y$ of size $n, m$ respectively and corresponding empirical measures $\widehat P_X, \widehat P_Y$, we define
\begin{equation}\label{eqn:mmd_u}
\begin{aligned}
    \mmd_u^2(\widehat P_X, \widehat P_Y) \eqdef \sum_{i \neq j} \frac{K(X_i,X_j)}{n(n-1)} +  \sum_{i\neq j} \frac{K(Y_i,Y_j)}{m(m-1)} - 2 \sum_{i,j} \frac{K(X_i,Y_j)}{mn}. 
\end{aligned}
\end{equation}
We also write $\langle \theta_{\widehat P_X}, \theta_{\widehat P_X}\rangle_{u,\cal H_K} \eqdef \|\theta_{\widehat P_X}\|_{u,\cal H_K}^2 \eqdef \frac{1}{n(n-1)}\sum_{i\neq j} K(X_i,X_j)$ and extend linearly. The $u$ stands for unbiased, since $\E \mmd_u^2(\widehat P_X, \widehat P_Y) = \mmd^2(P_X,P_Y) \neq \E \mmd^2(\widehat P_X, \widehat P_Y)$ in general.



\subsection{Test Statistic}\label{sec:the statistic}
With Section \ref{sec: MMDintro} behind us, we are in a position to define the test statistic that we use to tackle \eqref{eqn:lfht mix def}. Suppose that we have samples $X,Y,Z$ of sizes $n,n,m$ from the probability measures $P_X,P_Y,P_Z$. Write $\widehat P_X$ for the empirical measure of sample $X$, and analogously for $Y,Z$. The core of our test statistic for \eqref{eqn:lfht mix def} is the following:
\begin{align*}
    T(X,Y,Z) &\eqdef \langle \theta_{\widehat P_Z}, \theta_{\widehat P_Y}-\theta_{\widehat P_X} \rangle_{u, \cal H_K} = \frac{1}{nm}\sum_{i=1}^n\sum_{j=1}^m \Big\{K(Z_j,Y_i)-K(Z_j,X_i)\Big\}.  \numberthis\label{eq: teststats}
\end{align*}
Note that $T$ is of the additive form  
$\frac{1}{m}\sum_{j=1}^m f(Z_i)$ where $f(z)\eqdef \theta_{\widehat P_Y}(z)-\theta_{\widehat P_X}(z)$ can be interpreted as the \emph{witness function} of \cite{gretton2012kernel, jitkrittum2018informative}.
Given some $\pi \in [0,1]$ (taken to be half the predicted signal rate $\delta/2$ in our proofs), the output of our test is
\begin{equation}
    \Psi_{\pi} = \one\Big\{T(X,Y,Z) \geq \gamma(X,Y,\pi)\Big\}, 
    \label{eq: test} \text{ where }\gamma(X,Y,\pi) = \pi \mmd^2_u(\widehat P_X, \widehat P_Y) + T(X,Y,X). 
\end{equation}
The threshold $\gamma$ gives $\Psi_\pi$ a natural geometric interpretation: it checks whether the projection of $\theta_{\widehat P_Z}-\theta_{\widehat P_X}$ onto the vector $\theta_{\widehat P_Y}-\theta_{\widehat P_X}$ falls further than $\pi$ along the segment joining $\theta_{\widehat P_X}$ to $\theta_{\widehat P_Y}$ (up to deviations due to the omitted diagonal terms, see \Cref{sec:diagonal}).

Setting $\delta=1$ in \eqref{eqn:lfht mix def} recovers \eqref{eqn:LFHT def}, and the corresponding test output is $\Psi_{\delta/2}=\Psi_{1/2}=1$ if and only if $\mmd_u(\widehat P_Z, \widehat P_X) \geq \mmd_u(\widehat P_Z, \widehat P_Y)$. This very statistic (i.e. $\mmd_u(\widehat P_Z,\widehat P_X)-\mmd_u(\widehat P_X,\widehat P_Y)$) has been considered in the past for relative goodness-of-fit testing \cite{bounliphone2015test} where it's asymptotic properties are established. In the non-asymptotic setting, if MMD is replaced by the $L^2$-distance we recover the test statistic studied by \cite{huang2012classification,kelly2012classification,RY22SBI}. However, we are the first to introduce $\Psi_\delta$ for $\delta\neq1$ and to study MMD-based tests for (m)LFHT in a non-asymptotic setting. 

\textbf{Variance cancellation}\quad At first sight it may seem more natural to the reader to threshold the
distance $\mmd_u(\widehat P_Z, \widehat P_X)$, resulting in rejection if, say, $\mmd_u(\widehat P_Z, \widehat P_X) \geq \mmd_u(\widehat
P_X, \widehat P_Y)\delta/2$. The geometric meaning of this would be similar to the one
outlined above. However, there is a crucial difference: \eqref{eqn:LFHT def} (the case $\delta=1$) is possible with very little experimental data $m$ due to the
\textit{cancellation of variance}. More precisely, the statistic $\mmd^2(\widehat P_Z, \widehat
P_X)$ contains the term $\frac{1}{m(m-1)}\sum_{i\neq j}K(Z_i,Z_j)$ --- whose variance is
prohibitively large and would inflate the $m$ required for reliable testing --- but this can be
canceled by subtracting $\mmd^2(\widehat P_Z, \widehat P_Y)$. Our statistic
$T(X,Y,Z)-\gamma(X,Y,\pi)$ simply generalizes this idea to \eqref{eqn:lfht mix def}.


\section{Minimax Rates of Testing}\label{sec:rates}

\subsection{Upper Bounds on the Minimax Sample Complexity of \texorpdfstring{\eqref{eqn:lfht mix def}}{mLFHT}}\label{sec:upper bounds}
Let us start by reintroducing \eqref{eqn:lfht mix def} in a rigorous fashion. Given $C,\epsilon,R \geq 0$, let $\cal P_\mu(C,\epsilon,R)$ denote the set of triples $(P_X,P_Y,P_Z)$ of distributions such that the following three conditions hold:
\begin{enumerate}
    \vspace{-2mm}\item[$(i)$]  $P_X, P_Y$ and $P_Z$ have $\mu$-densities bounded by $C$, \label{item:mlfht1}
    \vspace{-1mm}\item[$(ii)$] $\mmd(P_X,P_Y) \geq \epsilon$, and\label{item:mlfht2}
    \vspace{-1mm}\item[$(iii)$] $\mmd(P_Z, (1-\nu) P_X+\nu P_Y) \leq R\cdot\mmd(P_X,P_Y)$\label{item:mlfht3}, 
\end{enumerate}
\vspace{-1mm}
where we define $\nu = \nu(P_X,P_Y,P_Z) = \arg\min_{\nu'\in\R}\mmd(P_Z,(1-\nu')P_X + \nu'P_Y)$. For some $\delta > 0$, consider the two hypotheses 
\begin{equation}
\begin{aligned}\label{eqn:mlfht rigorous def}
    &H_0(C,\epsilon, \delta, R): (P_X,P_Y,P_Z) \in \cal P_\mu(C,\epsilon,R) \text{ and } \nu=0 \\
    &H_1(C,\epsilon, \delta, R): (P_X,P_Y,P_Z) \in \cal P_\mu(C,\epsilon,R) \text{ and } \nu\geq\delta, 
\end{aligned}
\end{equation}
which we regard as subsets of probability measures. Notice that $R$ controls the level of mis-specification in the direction that is orthogonal to the line connecting the kernel embeddings of $P_X$ and $P_Y$. Setting $R=0$ simply asserts that $P_Z$ is guaranteed to be a mixture of $P_X$ and $P_Y$, as is the case for prior works on LFHT. Before presenting our main result on the minimax sample complexity of mLFHT, let us define one final piece of terminology. We say that a test $\Psi$, which takes some data as input and takes values in $\{0,1\}$, has total error probability less than $\alpha$ for the problem of testing $H_0$ vs $H_1$ if 
\begin{equation}\label{eqn:tot err prob}
    \sup_{P \in H_0} P(\Psi = 1) + \sup_{Q \in H_1} Q(\Psi=0) \leq \alpha. 
\end{equation}

\begin{theorem}\label{thm: MMDUpper}
Suppose we observe three i.i.d. samples $X,Y,Z$ from distributions $P_X,P_Y,P_Z$ composed of $n,n,m$ observations respectively and let $C \in (0,\infty)$ and $R,\epsilon \geq 0$ and $\delta \in (0,1)$. There exists a universal constant $c>0$ such that $\Psi_{\delta/2}$ defined \Cref{sec:the statistic} tests $H_0$ vs $H_1$, as defined in \eqref{eqn:mlfht rigorous def}, at total error $\alpha$ provided
    \begin{align*}
        \min\{m,n\} \geq c \frac{C\|\lambda\|_\infty\log(1/\alpha)}{\left(\epsilon\delta/(1+R)\right)^2} \qquad\text{ and }\qquad
        \min\{n, \sqrt{nm}\} \geq c \frac{C\|\lambda\|_2\log(1/\alpha)}{\delta\epsilon^2}. 
    \end{align*} 
\end{theorem}

\begin{wrapfigure}{R}{0.4\textwidth}
    \vspace{-6mm}
    \begin{center}
    \includegraphics[width=0.4\textwidth]{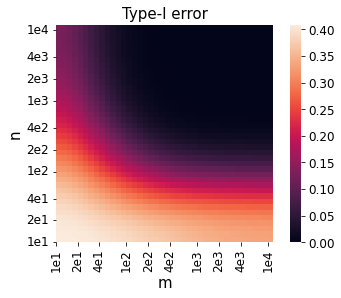}
    \vspace{-5mm}
    \caption{\small
    $n$ versus $m$ trade-off for the toy experiment, verifying \Cref{thm: MMDUpper}. Probabilities estimated over $10^4$ runs, and smoothed using Gaussian noise. }\label{fig:toy trade off}
    \end{center}
    \vspace{-5mm}
\end{wrapfigure}

Note that \Cref{thm: MMDUpper} does \emph{not} place assumptions on the distributions $P_X,P_Y$ beyond bounded density with respect to the base measure $\mu$. This is different from usual results in statistics, where prior specification of distribution classes is crucial. On the other hand, instead of standard distances such as $L^p$, we assume separation with respect to MMD and the latter is potentially harder to interpret than, say, $L^1$ i.e. total variation. We do point out that our \Cref{thm: MMDUpper} can be used to derive results in the classical setting; we discuss this further in Section \ref{ssec:connections}. 

In an appropriate regime of the parameters, the sufficient sample complexity in \Cref{thm: MMDUpper} exhibits a trade-off of the form $\min\{n, \sqrt{m n}\} \gtrsim \|\lambda\|_2\log(1/\alpha)/(\delta\epsilon^2)$ between the number of simulation samples $n$ and real observations $m$. This trade-off is shown in Figure \ref{fig:toy trade off} using data from a toy problem. The trade-off is clearly asymmetric and the relationship $m\cdot n\geq\textit{const.}$ also seems to appear. In this toy problem we set $R=0,\delta=1, \epsilon=.3,k=100$ and $P_X=P_Z, P_Y$ are distributions on $\{1,2,\dots,k\}$ with $P_X(i)=(1+\epsilon\cdot(2\cdot\one\{i\text{ odd}\}-1))/k = 2/k-P_Y(i)$ for all $i=1,2,\dots,k$. The kernel we take is $K(x,y)=\sum_{i=1}^k\one\{x=y=i\}$ and $\mu$ is simply the counting measure; the resulting MMD is simply the $L^2$-distance on pmfs. 

Figure \ref{fig:trade off} illustrates a larger scale experiment using real data using a trained kernel. Note that we plot the \emph{total} number $n$ of simulation samples, including those used for \emph{training} the kernel itself (see \cref{sec:learning kernels}); which ensures that Figure \ref{fig:trade off} gives a realistic picture of data requirements. However, due to the dependence between the kernel and the data, \Cref{thm: MMDUpper} no longer applies. Nevertheless, we observe a trade-off similar to Figure \ref{fig:toy trade off}.

\vspace{-1mm}
\subsection{Lower Bounds on the Minimax Sample Complexity of \texorpdfstring{\eqref{eqn:lfht mix def}}{mLFHT}}\label{sec:lower bounds}


In this section we prove a minimax lower bound on the sample complexity of mLFHT, giving a partial converse to \Cref{thm: MMDUpper}. Before we can state this results, we must make some technical definitions. Given $J\geq2$, let $\|\lambda\|_{2J}^2 \eqdef \sum_{j=2}^J \lambda_j^2$ and define 
\begin{equation*}
    J^\star_{\epsilon} \eqdef \max\Big\{J: \sup\limits_{\eta_j=\pm1} \Big\|\sum_{j=2}^J \eta_j\sqrt{\lambda_j}e_j\Big\|_\infty\leq \frac{\|\lambda\|_{2J}}{2\epsilon}\Big\}.
\end{equation*}

\begin{theorem}[Lower Bounds for mLFHT]\label{thm:lower bds}
Suppose that $\int_{\cal X}K(x,y)\mu(\D x)\equiv\lambda_1$, $\mu(\cal X)=1$ and $\sup_{x\in\cal X}K(x,x)\leq1$. There exists a universal constant $c>0$ such that any test of $H_0$ vs $H_1$, as defined in \eqref{eqn:mlfht rigorous def}, with total error at most $\alpha$ must use a number $(n,m)$ of observations that satisfy 
\begin{align*}
    m \geq c\frac{\lambda_2\log(1/\alpha)}{\epsilon^2 \delta^2} \quad\text{and}\quad
    n \geq c\frac{\|\lambda\|_{2J^\star_\epsilon}\sqrt{\log(1/\alpha)}}{\epsilon^2} \quad\text{and}\quad \delta m+ \sqrt{mn} &\geq c\frac{\|\lambda\|_{2J^\star_\epsilon}\sqrt{\log(1/\alpha)}}{\epsilon^2\delta}. 
\end{align*}
\end{theorem}

\begin{remark}
    Recall that the eigenvalues $\lambda$ \textit{depend on the choice of $\mu$}, so that by choosing a different base measure $\mu$ one can optimize the lower bound. However, since $P_X,P_Y,P_Z$ are assumed to have bounded density with respect to $\mu$, this appears rather involved. 
\end{remark}

\begin{remark}
The requirements $\sup_{x\in\cal X}K(x,x)\leq1$ and $\mu(\cal X)=1$ are is essentially without loss of generality, as $\mu$ and $K$ can be rescaled. The condition $\int_{\cal X}K(x,y)\mu(\D x)\equiv\lambda_1$ implies that the top eigenfunction $e_1$ is equal to a constant or equivalently, that $y\mapsto K(x,y)\mu(\D x)$ defines a Markov kernel up to a normalizing constant. 
\end{remark}

\subsection{Tightness of \Cref{thm: MMDUpper,thm:lower bds}}\label{ssec:tightness}

\textbf{Dependence on $\pmb{\|\lambda\|_2}$}\quad An apparent weakness of \Cref{thm:lower bds} is its reliance on the unknown value $J^\star_\epsilon$, which depends on the specifics of the kernel $K$ and base measure $\mu$. Determining it is potentially highly nontrivial even for simple kernels. Slightly weakening \Cref{thm:lower bds} we obtain the following corollary, which shows that the dependence on $\|\lambda\|_2$ is tight, at least for small $\epsilon$. 
\begin{corollary}\label{cor:lower bd}
    Suppose $J \geq 2$ is such that $\sum_{j=2}^J\lambda_j^2 \geq c^2 \|\lambda\|_2^2$ for some $c\leq1$. Then $\|\lambda\|_{2J^\star_\epsilon}$ can be replaced by $c\|\lambda\|_2$ in \Cref{thm:lower bds} whenever $\epsilon \leq c\|\lambda\|_2/(2\sqrt{J-1})$. 
\end{corollary}



\textbf{Dependence on $\pmb{R}$ and $\pmb{\alpha}$}\quad Due to the general nature of our lower bound constructions, it is difficult to capture the dependence on the misspecification parameter $R$. As for the probability of error $\alpha$, based on related work \cite{diakonikolas2021optimal} we expect the gap of size $\sqrt{\log(1/\alpha)}$ to be a shortcoming of \Cref{thm: MMDUpper} and not the lower bound. Closing this gap may require a different approach, however, as tests based on empirical $L^2$ distances are known to have a sub-optimal concentration \cite{huang2012classification}. 

\textbf{Dependence on $\pmb{\delta}$}\quad The correct dependence on the signal rate $\delta$ is the most important question left open by our theoretical results. Any method requiring $n$ larger than a function of $\delta$ irrespective of $m$ (as in \Cref{thm: MMDUpper}) is provably sub-optimal because taking $m\gtrsim1/(\delta\epsilon)^2$ and $n$ large enough to estimate both $P_X,P_Y$ to within accuracy $\epsilon/10$  always suffices to reach a fixed level of total error. 

\vspace{-1mm}
\subsection{Relation to Prior Results}\label{ssec:connections}

\vspace{-1mm}
In this section we discuss some connections of \Cref{thm: MMDUpper} to prior work. Specifically, we discuss how \Cref{thm: MMDUpper} recovers some known results in the literature \cite{arias2018remember,li2019optimality, RY22SBI} that are \emph{minimax optimal}. Details omitted in this section are included in Appendix \ref{sec:extra examples}. 

\textbf{Binary Hypothesis Testing}\label{sec: Binary}\quad
Suppose the two distributions $P_X,P_Y$ are \textit{known}, we are given $m$ i.i.d. observations $Z_1,\dots,Z_m \sim P_Z$ and our task is to decide between the hypotheses $H_0: P_X=P_Z$ versus $H_1: P_Y=P_Z$. Then, we may take $n=\infty, R=0, \delta=1$ in Theorem \ref{thm: MMDUpper} to conclude that 
\begin{equation*}
    m\geq c\cdot\frac{C \|\lambda\|_\infty\log(1/\alpha)}{\eps^2}
\end{equation*} 
observations suffice to perform the test at total error $\alpha$. 

\textbf{Two-Sample Testing}\quad 
Suppose we have two i.i.d. samples $X$ and $Y$, both of size $n$, from unknown distributions $P_X,P_Y$ respectively and our task is to decide between $H_0:P_X=P_Y$ against $H_1:\mmd(P_X,P_Y) \geq \epsilon$. We split our $Y$ sample in half resulting in $Y^{(1)}$ and $Y^{(2)}$ and form the statistic $\Psi_\sf{TS} \eqdef \Psi_{1/2}(X,Y^{(1)}, Y^{(2)}) - \Psi_{1/2}(Y^{(1)}, X, Y^{(2)})$, where $\Psi_{1/2}$ is defined in Section \ref{sec:the statistic}. Then $|\E \Psi_\sf{TS}|$ is equal to $0$ under the null hypothesis and is at least $1-2\alpha_1$ under the alternative, where $\alpha_1$ is the target total error probability of $\Psi_{1/2}$. Taking $\alpha_1=5\%$, by repeated sample splitting and majority voting we may amplify the success probability to $\alpha$ provided 
\begin{equation}\label{eqn:TS rate}
    n \geq c' \frac{C\|\lambda\|_2\log(1/\alpha)}{\epsilon^2}, 
\end{equation}
where $c'>0$ is universal (see Appendix for details). The upper bound \eqref{eqn:TS rate} partly recovers \cite[Theorem 3 and 5]{li2019optimality} where authors show that thresholding the MMD with Gaussian kernel $G_\sigma(x,y)=\sigma^{-d}\exp(-\|x-y\|^2/\sigma^2)$ achieves the minimax optimal sample complexity $n\asymp\epsilon^{-(2\beta+d/2)/\beta}$ for the problem of two-sample testing over the class $\cal P_{\beta,d}$ of $d$-dimensional $(\beta,2)$-Sobolev-smooth distributions (defined in Appendix \ref{sec:appendix sobolev example}) under $\epsilon$-$L^2$-separation. For this, taking $\sigma \asymp \epsilon^{1/\beta}$ ensures that $\|P-Q\|_{L^2}\lesssim\mmd(P, Q)$ over $P, Q\in \cal P_{\beta,d}$. Taking e.g. $\D\mu(x)=\exp(-\|x\|_2^2)\D x$ as the base measure, \eqref{eqn:TS rate} recovers the claimed sample complexity since $\|\lambda\|_2^2 = \int G^2_\sigma(x,y)\D \mu(x)\D \mu(y) =\cal O(\sigma^{-d})$ hiding dimension dependent constants. Our result requires a bounded density with respect to a Gaussian.

\textbf{Likelihood-Free Hypothesis Testing}\quad By taking $\alpha\asymp R \asymp \delta = \Theta(1)$ in Theorem \ref{thm: MMDUpper} we can recover many of the results in \cite{kelly2010universal,kelly2012classification,RY22SBI}. When $\cal X$ is finite, we can take the kernel $K(x,y) = \sum_{z \in \cal X} \one\{x=y=z\}$ in Theorem \ref{thm: MMDUpper} to obtain the results for bounded discrete distributions (defined in Appendix \ref{sec:appendix discrete example}) which state that under $\epsilon$-TV-separation the minimax optimal sample complexity is given by $m \gtrsim 1/\epsilon^2; \min\{n, \sqrt{nm}\} \gtrsim \sqrt{|\cal X|}/\epsilon^2$. A similar kernel recovers the optimal result for the class of $\beta$-H\"older smooth densities on the hypercube $[0,1]^d$ (see Appendix \ref{sec:appendix holder example}).

\textbf{Curse of Dimensionality}\quad Using the Gaussian kernel $G_\sigma$ as for two-sample testing above, one can conclude by Theorem \ref{thm: MMDUpper} that the required number of samples for \eqref{eqn:lfht mix def} over the class $\cal P_{\beta,d}$ under $\epsilon$-$L^2$-separation grows at most like $\Omega\left((\frac{c}{\epsilon})^{\Omega(d)}\frac{1}{\delta^2}\right)$ for some $c>1$, instead of the expected $\Omega\left((\frac{c}{\epsilon\delta})^{\Omega(d)}\right)$. This may be interpreted as theoretical support for the success of LFI in practice where signal and background can be rather different (cf. \cite[Figures 2-3]{baldi2014searching}) and the difficulty of the problem stems from the rate of signal events being small (i.e. $\epsilon \approx 1$ but $\delta\ll1$).

\vspace{-2mm}
\section{Learning Kernels from Data}\label{sec:learning kernels}

\vspace{-1mm}
Given a \emph{fixed} kernel $K$, our \Cref{thm: MMDUpper,thm:lower bds} show that the sample complexity is heavily dependent on the separation $\epsilon$ under the given $\mmd$ as well as the spectrum $\lambda=\lambda(\mu,K)$ of the kernel. Thus, to have good test performance we need to use a kernel $K$ that is well-adapted to the problem at hand. In practice, however, instead of using a fixed kernel it would be only natural to use part of the simulation sample to try to \emph{learn} a good kernel. 

In Sections \ref{sec:learning kernels} and \ref{sec: empirics} we report experimental results after \emph{training} a kernel parameterized by a neural network on part of the simulation data. In particular, due to the dependence between the data and the kernel, \Cref{thm: MMDUpper} doesn't directly apply. Our main contribution here is showing the existence of an asymmetric simulation-experimentation trade-off (cf. Figure \ref{fig:trade off} and also \Cref{sec:lfht}) even in this realistic setting. Figure \ref{fig:trade off} plots the \emph{total} number $n$ of simulations used, including those used for training, so as to provide a realistic view of the amount of data used. The experiments also illustrate that the (trained-)kernel-based statistic of \Cref{sec:the statistic} achieves state-of-the-art performance.


 \vspace{-2mm}
\subsection{Proposed Training Algorithm}
Consider splitting the data into three parts: $(X^\sf{tr}, Y^\sf{tr})$ is used for training (optimizing) the kernel; $(X^\sf{ev}, Y^\sf{ev})$ is used to evaluate our test statistic at test time; and $(X^\sf{cal}, Y^\sf{cal})$ is used for calibrating the distribution of the test statistic under the null hypothesis. We write $n_{s}=|X^{s}|=|Y^{s}|$ for $s\in\{\sf{tr}, \sf{ev}, \sf{cal}\}$.
Given the training data $X^\sf{tr}, Y^\sf{tr}$ with empirical measures $\widehat P_{X^\sf{tr}},\widehat P_{Y^\sf{tr}}$, we maximize the objective in 
\begin{align}
\label{eq:J}
    \widehat{J}(X^\sf{tr},Y^\sf{tr}; K)=\frac{\mmd_u^2(\widehat P_{X^\sf{tr}},\widehat P_{Y^\sf{tr}};K)}{\widehat{\sigma}(X^\sf{tr},Y^\sf{tr};K)},
\end{align}
which was introduced in \cite{sutherland2016generative}. Here $\widehat{\sigma}^2$ is an estimator of the variance of $\mmd^2_u(\widehat P_{X^\sf{tr}},\widehat P_{Y^\sf{tr}};K)$ and is defined in \Cref{sec:sigma defn}. 

Intuitively, the objective $J$ aims to separate $P_X$ from $P_Y$ while keeping variance low. For a heuristic justification of its use for \eqref{eqn:lfht mix def} see 
Appendix. 

In \Cref{alg:learn_deep_kernel} we describe the training and testing procedure, which produces unbiased $p$-values for \eqref{eqn:lfht mix def} when there is no misspecification ($R=0$ in \Cref{thm: MMDUpper}). During training, we use the Adam optimizer \cite{kingma2014adam} with stochastic batches. 

\begin{proposition}\label{prop: consistent}
When there is no misspecification ($R=0$ in \Cref{thm: MMDUpper}), Algorithm 1 outputs an unbiased estimate of the $p$-value that is consistent as $\min\{n_\sf{cal}, k\} \to\infty$. 
\end{proposition}


\begin{algorithm}[!t]
\footnotesize
\caption{mLFHT with a learned deep kernel}
\label{alg:learn_deep_kernel}
\begin{algorithmic}
\STATE \textbf{Input:} $(X^\sf{tr}$, $X^\sf{ev}, X^\sf{cal})$, $(Y^\sf{tr}, Y^\sf{ev}, Y^\sf{cal}$); parametrized kernel $K_\omega$; hyperparameters and initialization.

\vspace{0.5mm}
\STATE \textbf{\# Phase 1:}\textit{ Kernel training (optimization) on $X^\sf{tr}$ and $Y^\sf{tr}$.\hfill
}
\vspace{0.5mm}




\STATE\quad $\omega \gets \arg\max_{\omega}^{\textnormal{Optimizer}} \widehat{J} (X^\sf{tr}, Y^\sf{tr}; K_\omega)$; \hfill \textit{\# maximize objective as defined in \eqref{eq:J}}
\vspace{0.5mm}
\STATE \textbf{\# Phase 2:}\textit{ Distributional calibration of test statistic (under null hypothesis).}
\vspace{0.5mm}

\FOR{$r = 1, 2, \dots, k$}
\STATE $Z^{\sf{cal},r}\gets $ sample $m$ points without replacement from $X^\sf{cal}$;
\STATE $\mathit{T}_r \gets
\frac{1}{n_{\sf{ev}}m}\sum_{i,j} \Big( K_{\omega}(Z^{\sf{cal},r}_i,Y^\sf{ev}_j)-K_{\omega}(Z^{\sf{cal},r}_i,X^\sf{ev}_j)\Big)$;
\ENDFOR

\vspace{0.5mm}
\STATE \textbf{\# Phase 3:}\textit{ Inference with input $Z$}.
\vspace{0.5mm}
\STATE \quad $\widehat T\gets\frac{1}{n_{\sf{ev}}m}\sum_{i,j} \Big( K_{\omega}(Z_i,Y^\sf{ev}_j)-K_{\omega}(Z_i,X^\sf{ev}_j) \Big)$;

\STATE \textbf{Output: } Estimated $p$-value: $\frac{1}{k}\sum_{i=1}^k\one\{\widehat T < T_i\}$.

\end{algorithmic}

\end{algorithm}
\vspace{-2mm}

\textbf{Time complexity}\quad  \Cref{alg:learn_deep_kernel} runs in three separate stages: training, calibration, and inference. The first two take $O\left(\#\text{epochs}\cdot B^2+kn_{\sf{ev}}m\right)$ total time, where $B$ is the batch size,
whereas Phase 3 takes only $O(n_\sf{ev}m)$ time, which is generally much faster especially if $n_\sf{ev}<\!\!<n_\sf{tr}$.

\textbf{Sample usage}\quad Empirically, data splitting in \cref{alg:learn_deep_kernel} can have non-trivial effects on performance. Instead of training the kernel on only a fraction of the data ($\{X^{\sf{tr}}, Y^\sf{tr}\} \cap \{X^{\sf{ev}}, Y^\sf{ev}\} =\varnothing$), we discovered that taking $ \{X^{\sf{ev}},Y^{\sf{ev}}\}\subseteq\{X^{\sf{tr}},Y^{\sf{tr}}\}$ results in more efficient use of data. The stronger condition $\{X^{\sf{ev}},Y^{\sf{ev}}\}=\{X^{\sf{tr}},Y^{\sf{tr}}\}$ can also be applied; we take $\subseteq$ to reduce time complexity. We do, however, crucially require $X^\sf{cal}, Y^\sf{cal}$ in Phase 2 to be independently sampled (``held-out'') for consistent $p$-value estimation. Finally, we remark also that splitting this way is only valid in the context of \Cref{alg:learn_deep_kernel}. For the test \eqref{eq: test} using the data-dependent threshold $\gamma$, one needs $\{X^{\sf{tr}}, Y^\sf{tr}\} \cap \{X^{\sf{ev}}, Y^\sf{ev}\} =\varnothing$ to estimate $\gamma$.

\vspace{-2mm}
\subsection{Classifier-Based Tests and Other Benchmarks}\label{sec: Classifier}
Let $\phi:\cal X \to [0,1]$ be a classifier, assigning small values to $P_X$ and high values to $P_Y$ by minimizing cross-entropy loss of a classifier net. There are two natural test statistics based on $\phi$: 

\textbf{$\Sch$'s Test.} The first idea, attributed to $\Sch$ in folklore \cite[Section 6]{devroye2012combinatorial}, is to take the statistic $T(Z)=\frac1m \sum_{i=1}^m\one\{\phi(Z_i)>t\}$ where $t$ is some (learn-able) threshold. 


\textbf{Approximate Neyman-Pearson / Logit Methods.}
If $\phi$ is trained to perfection, then $\phi(z)=\P(P_X|z)$ would be the likelihood and $\phi(z)/(1-\phi(z))$ would equal precisely the likelihood ratio between $P_Y$ and $P_X$ at $z$. This motivates the use of $T(Z) = \frac1m \sum_{i=1}^m \log(\phi(Z_i)/(1-\phi(Z_i)))$. See also \cite{cheng2019classification}. 

Let us list the testing procedures that we benchmark against each other in our experiments.
\begin{enumerate}
    \vspace{-2mm}\item \textbf{MMD-M}: The $\mmd$ statistic \eqref{eq: teststats} using $K$ with the \underline{m}ixing architecture
    \begin{align*}
        K(x, y)=[(1-\tau)G_{\sigma}(\varphi_\omega(x), \varphi_\omega(y))+\tau]
        \cdot G_{\sigma_0}(x+\varphi'_{\omega'}(x), y+\varphi'_{\omega'}(y)). 
    \end{align*}
    Here $G_{\sigma}$ is the Gaussian kernel with variance $\sigma^2$; $\varphi_\omega, \varphi'_{\omega'}$ are NN's (with parameters $\omega,\omega'$), and $\sigma,\sigma_0,\tau,\omega,\omega'$ are trained.
    
    \vspace{-1mm}\item \textbf{MMD-G}: The $\mmd$ statistic \eqref{eq: teststats} using the \underline{G}aussian kernel architecture $K(x, y)=G_{\sigma}(\varphi_\omega(x), \varphi_\omega(y))$ where $\varphi_\omega$ is the feature mapping parametrized by a trained network and $\sigma$ is a trainable parameter.
    
    \vspace{-1mm}\item \textbf{MMD-O}: The $\mmd$ statistic \eqref{eq: teststats} using the Gaussian Kernel $K(x, y)=G_{\sigma}(x, y)$ with \underline{o}ptimized bandwidth $\sigma$. First proposed in \cite{bounliphone2015test, liu2020learning}.

    \vspace{-1mm}\item \textbf{UME}: An interpretable model comparison algorithm proposed by \cite{jitkrittum2018informative}, which evaluates the kernel mean embedding on a chosen ``witness set''.

    \vspace{-1mm}\item \textbf{SCHE}, \textbf{LBI}: \underline{Sche}ff\'e's test and \underline{L}ogit \underline{B}ased \underline{I}nference methods \cite{cheng2019classification}, based on a binary classifier network $\phi$ trained via cross-entropy, introduced above.

    \vspace{-1mm}\item \textbf{RFM}: \underline{R}ecursive \underline{F}eature \underline{M}achines, a recently proposed kernel learning algorithm by \cite{radhakrishnan2022feature}.
\end{enumerate}

\vspace{-1mm}
\subsection{Additive Statistics and the Thresholding Trick}\label{sec:threshold trick}

\vspace{-1mm}
Given a function $f$ (usually obtained by training) and test data $Z=(Z_1,\dots,Z_m)$, we call a test additive if its output is obtained by thresholding $T_f(Z) \eqdef \frac1m \sum_{i =1}^m f(Z_i)$. We point out that all of \textbf{MMD-M/G/O}, \textbf{SCHE}, \textbf{LBI}, \textbf{UME}, \textbf{RFM} are of this form, see the Appendix for further details. Similarly to \cite{liu2020learning}, we observe that any such statistic can be realized by our kernel-based approach. 

\begin{proposition}
\label{prop:exists_kernel}
The kernel-based statistic defined in \eqref{eq: teststats} with the kernel $K(x, y)=f(x)f(y)$ is equal to $T_f$ up to a multiplicative and additive constant independent of $Z$.
\end{proposition}

\vspace{-1mm}
Motivated by the $\Sch$'s test, instead of directly thresholding the additive statistic $T_f(Z)$, we found empirically that replacing $f$ by $f_t(x) \eqdef \one\{f(x) > t\}$ can yield improved power. We set $t$ by maximizing an estimate of the significance under the null using a normal approximation, i.e. by solving $t_\sf{opt} \eqdef \argmax_t \frac{T_{f_t}(Y^{\sf{opt}})-T_{f_t}(X^{\sf{opt}})}{\sqrt{T_{f_t}(X^{\sf{opt}})(1-T_{f_t}(X^{\sf{opt}}))}}$, where $X^{\sf{opt}}, Y^{\sf{opt}}$ satisfy $\{X^{\sf{opt}}, Y^{\sf{opt}}\}\cap(\{X^{\sf{cal}}, Y^{\sf{cal}}\}\cup\{X^{\sf{ev}}, Y^{\sf{ev}}\})=\varnothing$. This trick improves the performance of our tester on the Higgs dataset in \Cref{sec: Higgs} but not for the image detection problem in \Cref{sec:image source}.

\vspace{-3mm}
\section{Experiments}\label{sec: empirics}
Our code can be found at \url{https://github.com/Sr-11/LFI}.
\begin{figure}[!t]
   \centering
    
   \includegraphics[scale=0.52]{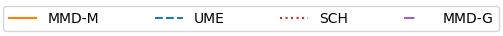}
   \hspace{-3mm} \includegraphics[scale=0.36]{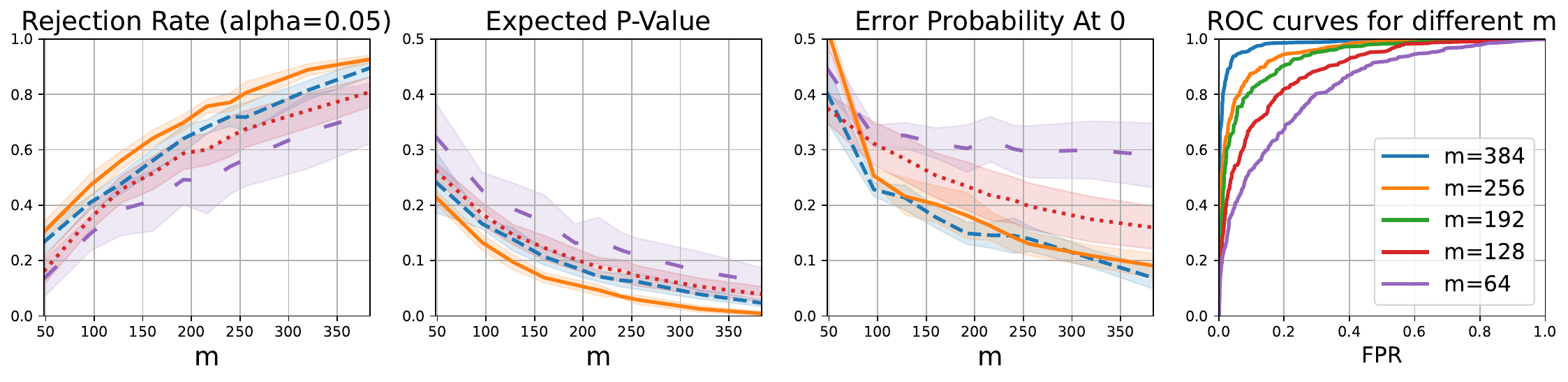}
   \vspace{-3mm}
    \caption{\small Empirical performance on \eqref{eqn:cifar hypotheses} for the CIFAR detection problem when $n_\sf{tr}
        =1920$. Plots from left to right are as follows. (a) rejection rate under the alternative if test rejects whenever the estimated $p$-value is smaller than $5\%$; (b) expected $p$-value \cite{EPV} under the alternative; (c) the average of type-I and II error probabilities when thresholded at 0 (different from \eqref{eq: test}, see Appendix); and (d) ROC curves for different $m$ using MMD-M and \Cref{alg:learn_deep_kernel}. Shaded
    area shows the standard deviation across $10$ independent runs. Missing benchmarks (thresholded MMD, MMD-O, LBI, RFM) are weaker; see Appendix for full plot.}
    \label{fig:cifar main}
    \vspace{-5mm}
\end{figure}
\vspace{-1mm}

\vspace{-1mm}
\subsection{Image Source Detection}\label{sec:image source}
\vspace{-1mm}
Our first empirical study looks at the task of detecting whether images come from the CIFAR-10 \cite{krizhevsky2020cifar} dataset or a SOTA generative model (DDPM) \cite{ ho2020denoising, von-platen-etal-2022-diffusers}. While source detection is on its own interesting, it turns out that detecting whether a group of images comes from the generative model versus the real dataset can be too ``easy'' (see experiments in \cite{jitkrittum2018informative}). Therefore, we consider a \emph{mixed alternative}, where the alternative hypothesis is not simply the generative model but CIFAR with planted DDPM images. Namely, our $n$ labeled images come from the following distributions:
\begin{equation}\label{eqn:cifar hypotheses}
P_X=\text{CIFAR}, \quad\text{and}\quad P_Y=\frac{1}{3}\cdot \text{DDPM}+\frac{2}{3}\cdot \text{CIFAR}.
\end{equation}
The goal is to test whether the $m$ unlabeled observations $Z$ have been corrupted with $\rho$ or more fraction of DDPM images (versus uncorrupted CIFAR); this corresponds to \eqref{eqn:LFHT def} (or equivalently \eqref{eqn:lfht mix def} with $\delta=1$).
\Cref{fig:cifar main} shows the performance of our approach with this mixed alternative. 

\textbf{Network Architecture}\quad 
With a standard deep CNN, the difference is only at the final layer: for the kernel-based tests it is a feature output; for classifiers, we add an extra linear layer to logits.

We see from Figure \ref{fig:cifar main} that our kernel-based test outperforms other benchmarks at a fixed training set size $n_{tr}$. One potential cause is that $\mmd$ has an ``optimization'' subroutine (which it solves in closed form) as it is an IPM. This additional layer of optimization may lead to better performance at small sample sizes. The thresholding trick does not seem to improve power empirically. We omit several benchmarks from this figure for graphic presentation and they do not exhibit good separating power; see the Appendix for the complete results. The bottom plot of Figure \ref{fig:trade off} shows $m$ and $n$ on $\log$-scale against the total probability of error, exhibiting the simulation-experimentation trade-off.

\vspace{-3mm}
\subsection{Higgs-Boson Discovery} \label{sec: Higgs} 

\begin{wrapfigure}{R}{0.5\textwidth}
    \vspace{-8mm}
    \begin{center}
    \vspace{-5mm}
    \includegraphics[width=0.5\textwidth]{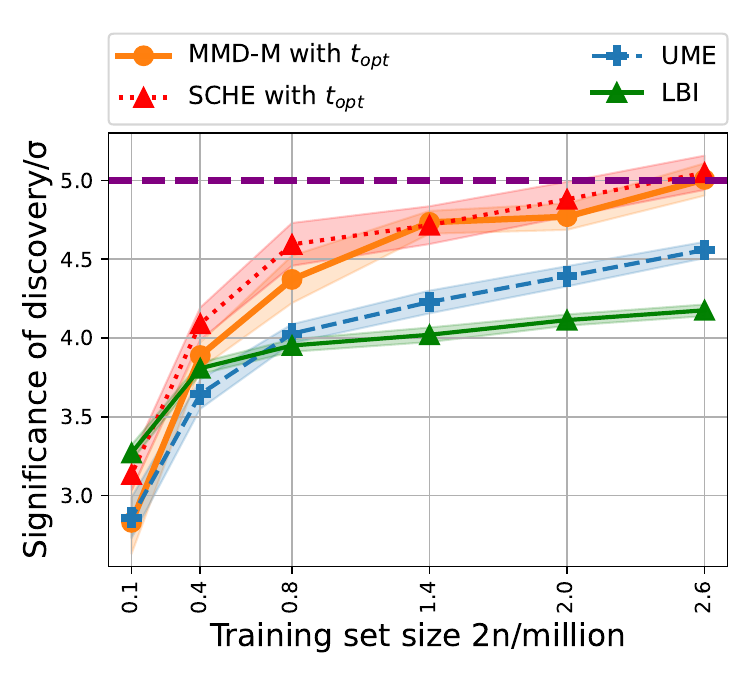}
    \vspace{-3mm}
    \caption{\small
    Expected significance of discovery on a mixture of 1000 backgrounds and
    100 signals in the Higgs experiment. Shaded area shows the standard deviation over $10$ independent runs. See Appendix for full plot including missing benchmarks. }
    \label{fig:higgs_p-n}
    \end{center}
\end{wrapfigure}

\vspace{-1mm}
The 2012 announcement of the Higgs boson's discovery by the ATLAS and CMS experiments \cite{aad2012observation,chatrchyan2012observation} marked a significant milestone in physics. The statistical problem inherent in the experiment is well-modeled by \eqref{eqn:lfht mix def}, using a signal rate $\delta$ predicted by theory and misspecification parameter $R=0$ (as was assumed in the original discovery). We consider our algorithm's power against past studies in the physics literature
\cite{baldi2014searching} as measured by the \emph{significance of discovery}. We note an important distinction from \cref{alg:learn_deep_kernel} in this application.

\textbf{Estimating the Significance}\quad In physics, the threshold for claiming a ``discovery'' is usually at a significance of $5\sigma$, corresponding
to a $p$-value of $2.87\times 10^{-7}$. Approximately $n_\sf{cal}\sim (2.87)^{-1} \times 10^7$ samples would be necessary for \cref{alg:learn_deep_kernel} to reach such a precision. Fortunately the distribution of the test statistic is approximated by a Gaussian customarily. We adopt this approach for our experiment hereby assuming that $m$ is large enough for the CLT to apply. We use the ``expected significance of discovery'' as our metric \cite{baldi2014searching} which, for the additive statistic $T_f=\frac1m\sum_{i=1}^m f(Z_i)$, is given by $\frac{\delta(T_f(Y^\sf{cal})-T_f(X^\sf{cal}))}{\sqrt{\widehat{\var}(f(X^{\sf{cal}}))/m}}$. If the thresholding trick (\Cref{sec:threshold trick}) is applied we use the more precise Binomial tail, in which case the significance is estimated by $-\Phi^{-1}(\P(\textup{Bin}(m,
T_{f_{t_\sf{opt}}}(X^\sf{cal})) \geq T_{f_{t_\sf{opt}}}(Z)))$, where $\Phi$ is the standard normal CDF. 

\textbf{Newtowk Architecture}\quad 
The architecture is a 6-layer feedforward net similar for all tests (kernel-based and classifiers) except for the last layer. We leave further details to the Appendix. 

As can be seen in \Cref{fig:higgs_p-n},
$\Sch$'s test and MMD-M with threshold $t_\sf{opt}$ are the best methods, achieving similar performance as the algorithm of \cite{baldi2014searching}; reaching the significance level of $5\sigma$ on 2.6 million simulated datapoints and a test sample made up of a mixture of 1000 backgrounds and 100 signals. The top plot of Figure \ref{fig:trade off} shows $m$ and $n$ on $\log$-scale against the total probability of error through performing the test \eqref{eq: test}, exhibiting the asymmetric simulation-experimentation trade-off. 

\vspace{-4mm}
\section{Conclusion}
\vspace{-3mm}
In this paper, we introduced \eqref{eqn:lfht mix def} as a theoretical model of real-world likelihood-free signal detection problems arising in science. We proposed a kernel-based test statistic and analyzed its minimax sample complexity, obtaining both upper (\Cref{thm: MMDUpper}) and lower bounds (\Cref{thm:lower bds}) in terms of multiple problem parameters, and discussed their tightness (\Cref{ssec:tightness}) and connections to prior work (\Cref{ssec:connections}). On the empirical side, we described a method for training a parametrized kernel and proposed a consistent $p$-value estimate (\Cref{alg:learn_deep_kernel} and \Cref{prop: consistent}). We examined the performance of our method in two experiments and found that parametrized kernels achieve state-of-the-art performance compared to relevant benchmarks from the literature. Moreover, we confirmed experimentally the existence of the asymmetric simulation-experimentation trade-off (Figure \ref{fig:trade off}) which is suggested by minimax analysis. We defer further special cases of \Cref{thm: MMDUpper}, all relevant proofs and experimental details to the Appendix.

\begin{ack}
    YP was supported in part by the National Science
    Foundation under Grant No CCF-2131115. Research was sponsored by the United States Air Force Research Laboratory and the Department of the Air Force Artificial Intelligence Accelerator and was accomplished under Cooperative Agreement Number FA8750-19-2-1000. The views and conclusions contained in this document are those of the authors and should not be interpreted as representing the official policies, either expressed or implied, of the Department of the Air Force or the U.S. Government. The U.S. Government is authorized to reproduce and distribute reprints for Government purposes notwithstanding any copyright notation herein.
\end{ack}

\bibliography{ref.bib}

\appendix\onecolumn
\addtocontents{toc}{\protect\setcounter{tocdepth}{2}}
\renewcommand{\contentsname}{Table of Contents}
\newpage
\addcontentsline{toc}{section}{Appendix}
\part{Appendix}
\parttoc


\section{Notation}
We use $A\gtrsim B, A \lesssim B, A \asymp B$ to denote $A = \Omega(B), B = \Omega(A)$ and $A = \Theta(B)$ respectively, where the hidden constants depend on untracked parameters multiplicatively.\footnote{For example, the first equation in \eqref{eq:Pdb upper} means that there exists a constant $c$ independent of $\alpha, k, \eps, \delta, R$, such that $\min\{m,n\}\geq c\frac{\log(1/\alpha)(1+R)^2}{k\epsilon^2\delta^2}$.}

We write $\TV,\KL,\chi^2$ for total-variation, KL-divergence and $\chi^2$-divergence, respectively. We write $D(P_{Y|X}\|Q_{Y|X}|P_X) = \E_{X \sim P_X} D(P_{Y|X} \| Q_{Y|X})$ as the \emph{conditional divergence} for any probability measures $P,Q$ on two variables $X,Y$ and divergence $D \in \{\TV,\KL,\chi^2\}$. 

We write $\ell^p$ for the usual $\ell^p$ sequence space and $L^p$ for the usual $L^p$ space with respect to the Lebesgue measure. Both the $\ell^p$ norm and the $L^p$ norm are written as $\|\cdot\|_p$ if no ambiguity arises.

For real numbers $a,b\in\R$ we also write $\max\{a,b\}$ as $a\vee b$ and $\min\{a,b\}$ as $a\wedge b$.

We use $\vec{1}_d$ to denote an $d$-dimensional all $1$'s vector.

For an integer $k\in\mathbb{Z}^+$, we write $[k]$ as a short notation for the set $\{1,2,\dots,k\}$.

In the proofs of \Cref{thm: MMDUpper} and \Cref{thm:lower bds}, we use $\stackrel{!}{=}$ for an equality that we are trying to prove.

\section{Applications of Theorem \ref{thm: MMDUpper}}\label{sec:extra examples}
Usually, minimax rates of testing are proven under separation assumptions using more traditional
measures of distance such as $L^p$, where $p\in[1,\infty]$. In this section we show one example of how
\Cref{thm: MMDUpper} can be used to recover known results, and also obtain some novel results under $L^2$-separation and $L^1$-separation.

\subsection{Bounded Discrete Distributions Under \texorpdfstring{$L^2/L^1$}{L2-L1}-Separation}\label{sec:appendix discrete example}

\textbf{Sample Complexity Upper Bounds}\quad Let $\cal P_{\sf{Db}}(k, C)$ be the set of all discrete distributions $P$ supported on $[k]=\{1,2,\dots,k\}$ satisfying $\max_{1\leq i\leq k} p(i) \leq C/k$, where $p$ is the probability mass function of $P$ (here $\sum_{i=1}^k p(k)=1$). For distributions $P_X,P_Y,P_Z$ we shall write $p_X,p_Y,p_Z$ as their probability mass functions, respectively.

Let us apply \Cref{thm: MMDUpper} with underlying space $\cal X = [k]$ and measure $\mu = \frac1k \sum_{i=1}^k \delta_i$. Take the kernel $K(x,y) =\one\{x=y\}= \sum_{i=1}^k \one\{x=y=i\}$, and note that for any two distributions $P_X,P_Y$ we have 
$$\mmd^2(P_X,P_Y)=\E\Big[K(X,X')+K(Y,Y')-2K(X,Y)\Big]=\sum_i|p_X(i)-p_Y(i)|^2$$
where $(X,X',Y,Y')\sim P_X^{\otimes2}\otimes  P_Y^{\otimes2}$. So the corresponding MMD is the $\ell^2$-distance on probability mass functions. Note also that $K=\sum_{i=1}^k \frac{1}{k} \left(\sqrt{k}\one\{x=i\}\right)\left(\sqrt{k}\one\{y=i\}\right)$, where $\left\{\sqrt{k}\one\{x=i\}\right\}_{i=1}^k$ forms an orthonormal basis of $L^2(\mu)$. So $K$ has only one nonzero eigenvalue, namely 
$$\lambda_1=\lambda_2=\ldots=\lambda_k=1/k,$$
of multiplicity $k$. Suppose that we observe samples $X,Y,Z$ of size $n,n,m$ from $P_X,P_Y,P_Z \in \cal P_\sf{Db}(k,C)$, where $\mmd(P_X,P_Y) = \sqrt{\sum_i|p_X(i)-p_Y(i)|^2} \geq \epsilon$.
Plugging into Theorem \ref{thm: MMDUpper} shows that:

\begin{proposition}
For any two $P_X, P_Y\in \cal P_{\sf{Db}}(k, C)$, if the $\ell^2$-distance between $p_X, p_Y$ is at least $\epsilon$, then testing \eqref{eqn:lfht mix def} is possible at total error $\alpha$ using $n$ simulation samples and $m$ real data samples provided that
\begin{equation}\label{eq:Pdb upper}
    \begin{aligned}
    \min\{m,n\} &\gtrsim \frac{C \|\lambda\|_\infty \log(1/\alpha)(1+R)^2}{\delta^2\epsilon^2} \asymp \frac{\log(1/\alpha)(1+R)^2}{k\epsilon^2\delta^2},\\
    \min\{n,\sqrt{mn}\} &\gtrsim \frac{C \|\lambda\|_2 \sqrt{\log(1/\alpha)}}{\epsilon^2\delta} \asymp \frac{\sqrt{\log(1/\alpha)}}{\sqrt{k}\epsilon^2\delta}.
    \end{aligned}
\end{equation}
where $R$ is defined as in the assumption $(iii)$ of \Cref{item:mlfht3}.
\end{proposition}


We can convert the above results to measure separation with respect to total variation (recall $\TV(p,q) = \frac12\sum_i|p(i)-q(i)|=\frac12 \|p-q\|_1$) using the AM-QM inequality $\|p_X-p_Y\|_1 \leq \sqrt k\|p_X-p_Y\|_2$. Then, taking $R\asymp\alpha\asymp\delta=\Theta(1)$ recovers the minimax optimal results of \cite{kelly2010universal,kelly2012classification,RY22SBI}, for $\LF$ over the class $\cal P_\sf{Db}$. Note that analogous results for two-sample testing follow from the above using the reduction presented in \Cref{ssec:connections}.

\textbf{Sample Complexity Lower Bounds}\quad Recall the definition of $J^\star_\epsilon$ and note that $\|\lambda\|_{2,J}^2=\frac{\min(J-1,k)}{k^2}$ for all $J \geq 2$. By \Cref{cor:lower bd} we see that $J^\star_\epsilon \gtrsim k$ as soon as $\epsilon \lesssim 1/k$. Thus, for $\epsilon \lesssim 1/k$ the necessity of
\begin{equation}\label{eqn:P_db lower}
    m\gtrsim\frac{\log(1/\alpha)}{k\epsilon^2\delta^2},\, n\gtrsim \frac{\sqrt{\log(1/\alpha)}}{\sqrt{k}\epsilon^2}\,\text{ and }\, m+\sqrt{mn} \gtrsim \frac{\sqrt{\log(1/\alpha)}}{\sqrt k\epsilon^2\delta}
\end{equation}
follows by \Cref{thm:lower bds}. Here it is crucial to note that when $\delta=\Theta(1)$, we have
\begin{align*}
    m+\sqrt{mn} \gtrsim \frac{\sqrt{\log(1/\alpha)}}{\sqrt k\epsilon^2} \text{ and } n\gtrsim \frac{\sqrt{\log(1/\alpha)}}{\sqrt{k}\epsilon^2}
    \iff
    \sqrt{mn} \gtrsim \frac{\sqrt{\log(1/\alpha)}}{\sqrt k\epsilon^2} \text{ and } n\gtrsim \frac{\sqrt{\log(1/\alpha)}}{\sqrt{k}\epsilon^2}
\end{align*}
and hence the upper bound \eqref{eq:Pdb upper} meets with the lower bound \eqref{eqn:P_db lower} provided $R\asymp \delta=\Theta(1)$. Once again, setting $R\asymp \delta\asymp\alpha=\Theta(1)$ we the optimal lower bounds recovering the results of \cite{RY22SBI} (in the regime $\epsilon \lesssim 1/k$). In short we can also recover the following result for $\LF$.
\begin{proposition}[{{\cite[Theorem 1, adapted]{RY22SBI}}}]
    On the class $P_{\sf{Db}}(k, C)$, using $n$ simulation samples and $m$ real data samples, if
    \begin{align}\label{eq:Pdb3}
        n\gtrsim \frac{1}{\sqrt{k}\epsilon^2},\quad m\gtrsim \frac{1}{k\epsilon^2}, \quad \sqrt{mn}\gtrsim\frac{1}{\sqrt{k}\eps^2},
    \end{align}
    then  for any two distributions $P_X, P_Y \in \cal P_{\sf{Db}}(k, C)$ with $\|p_X-p_Y\|_2\geq\eps$, testing \eqref{eqn:LFHT def} is possible with a total error of $1\%$. 
    Conversely, to ensure the existence of a procedure that can test \eqref{eqn:LFHT def} with a total error of $1\%$ for any $P_X, P_Y \in \cal P_{\sf{Db}}(k, C)$ with $\|p_X-p_Y\|_2\geq\eps$, the number of observations $(n,m)$ must satisfy 
    \begin{align}\label{eq:Pdb4}
        n\gtrsim \frac{1}{\sqrt{k}\epsilon^2},\quad m\gtrsim \frac{1}{k\epsilon^2},\quad \sqrt{mn}\gtrsim\frac{1}{\sqrt{k}\eps^2}.
    \end{align}
    The implied constants in \eqref{eq:Pdb3} and \eqref{eq:Pdb4} do not depend on $k$ and $\eps$, but may differ.
\end{proposition}


\subsection{\texorpdfstring{$\beta$-H\"older Smooth Densities on $[0,1]^d$  Under $L^2/L^1$-Separation}{beta-Holder Smooth Densities on [0,1]^d  Under L^2/L^1-Separation}}\label{sec:appendix holder example}
\textbf{Sample Complexity Upper Bounds}\quad Let $\cal P_{\sf{H}}(\beta, d, C)$ be the set of all distributions on $[0,1]^d$ with $\beta$-H\"older smooth Lebesgue-density $p$ satisfying $\|p\|_{\cal C^\beta} \leq C$ for some constant $C>1$, where
\begin{align*}
&\|p\|_{\cal C^\beta} \eqdef \max_{0\leq |\alpha| \leq \lceil\beta-1\rceil} \|f^{(\alpha)}\|_\infty + 
\sup_{x\neq y \in [0,1]^d, |\alpha|=\lceil\beta-1\rceil} \frac{|f^{(\alpha)}(x)-f^{(\alpha)}(y)|}{\|x-y\|^{\beta-\lceil\beta-1\rceil}_2}, 
\end{align*}
where $\lceil \beta-1\rceil$ is the largest integer strictly smaller than $\beta$ and $|\alpha| =\sum_i \alpha_i$ is the norm of a multi-index $\alpha \in \N^d$. Abusing notation, we also use $\cal P_{\sf{H}}(\beta, d, C)$ to denote the set of all corresponding density functions.

We take $K(x,y) = \sum_j \one\{x ,y \in B_j\}$, where $\{B_j\}_{j\in[\kappa]^d}$ is the $j$'th cell of the regular grid of size $\kappa^d$ on $[0,1]^d$, i.e., $B_j = [(j-\vec{1}_d)/\kappa, j/\kappa]$ for $j \in [\kappa]^d$. Clearly there are $\kappa^d$ nonzero eigenvalues, each equal to $1$. The following approximation result is due to Ingster \citep{IngsterGoF}, see also \cite[Lemma 7.2]{arias2018remember}. 
\begin{lemma}\label{lem: smoothdensity}
    Let $f,g \in \cal P_\sf{H}(\beta, d, C)$ with $\|f-g\|_2 \geq \epsilon$. Then, there exist constants $c,c'$ independent of $\epsilon$ such that for any $\kappa\geq c\epsilon^{-1/\beta}$, 
    \begin{equation*}
        \mmd(f,g) \geq c' \|f-g\|_2. 
    \end{equation*}
\end{lemma}

Now, suppose that we have samples $X,Y,Z$ of size $n,n,m$ from $P_X,P_Y,P_Z \in \cal P_\sf{H}(\beta, d, C)$ with densities $p_X,p_Y,p_Z$ such that $\|p_X-p_Y\|_2 \geq \epsilon$. Then, \Cref{thm: MMDUpper} combined with \Cref{lem: smoothdensity} and the choice $\kappa \asymp \epsilon^{-1/\beta}$ shows that
\begin{proposition}
    Testing \eqref{eqn:lfht mix def} on  $\cal P_{\sf{H}}(\beta, d, C)$ at total error $\alpha$ using $n$ simulation and $m$ real data samples is possible provided
\begin{align*}
    \min\{m,n\}&\gtrsim \frac{C \|\lambda\|_\infty \log(1/\alpha)(1+R)^2}{\delta^2\epsilon^2} \asymp \frac{\log(1/\alpha)(1+R)^2}{\delta^2\epsilon^2}, \\
    \min\{n, \sqrt{nm}\} &\gtrsim \frac{C \|\lambda\|_2 \sqrt{\log(1/\alpha)}}{\epsilon^2\delta}  \asymp \frac{\sqrt{\log(1/\alpha)}}{\epsilon^{(2\beta+d/2)/\beta}\delta},
\end{align*}
where $\epsilon$ is an $L^2$-distance lower bound between $P_X, P_Y$ and $R$ is defined as in the assumption $(iii)$ of \Cref{item:mlfht3}.
\end{proposition}

Setting $R\asymp\alpha \asymp \delta = \Theta(1)$ recovers the optimal results of \cite{RY22SBI} for the class $\cal P_\sf{H}$. Once again, identical results under $L^1$ separation follow from Jensen's inequality $\|\cdot\|_{L^1([0,1]^d)} \leq \|\cdot\|_{L^2([0,1]^d)}$. Note that analogous results for two-sample testing follow from the above using the reduction presented in \Cref{ssec:connections}.

\textbf{Sample Complexity Lower Bounds}\quad 
The kernel defined in the previous paragraph is not suitable for constructing lower bounds over the class $\cal P_\sf{H}$ because its eigenfunctions do not necessarily lie in $\cal P_\sf{H}$. It would be possible to consider a different kernel that is more adapted to this problem/class but we do not pursue this here. 


\subsection{\texorpdfstring{$(\beta,2)$-Sobolev Smooth Densities on $\R^d$ Under $L^2$-Separation}{(beta,2)-Sobolev Smooth Densities on R^d Under L^2-Separation} }\label{sec:appendix sobolev example}

\textbf{Sample Complexity Upper Bounds}\quad Let $\cal P_\sf{S}(\beta, d, C)$ be the class of distributions that are supported on $\R^d$ and whose Lebesgue density $p$ satisfies $\|p\|_{\beta,2} \leq C$, where
\begin{equation}
    \|p\|_{\beta,2} \eqdef\left\|(1+\|\cdot\|)^\beta\cal F[p]\right\|_2
\end{equation}
and $\cal F$ denotes the Fourier transform. Again, abusing notation, we write $\cal P_\sf{S}(\beta, d, C)$ both as the set of distributions and the set of density functions. 

We take the Gaussian kernel $G_\sigma(x,y) = \sigma^{-d}\exp(-\|x-y\|_2^2/\sigma^2)$ on $\cal X = \R^d$ with base measure $\D\mu(x)=\exp(-x^2)\D x$. In \cite{li2019optimality} the authors showed that the two-sample test that thresholds the Gaussian MMD with appropriately chosen variance $\sigma^2$ achieves the minimax optimal sample complexity over $\cal P_\sf{S}$, when separation is measured by $L^2$. A key ingredient in their proof is the following inequality. 
\begin{lemma}[{{\cite[Lemma 5]{li2019optimality}}}]\label{lem:gauss kernel approx}
Let $f,g \in \cal P_\sf{S}(\beta, d, C)$ with $\|f-g\|_2 \geq \epsilon$. Then, there exist constants $c,c'$ independent of $\epsilon$ such that for any $\sigma \leq c\, \epsilon^{1/\beta}$, we have
\begin{equation*}
    \mmd(f,g) \geq c'\|f-g\|_2. 
\end{equation*}
\end{lemma}

Now, suppose that we have samples $X,Y,Z$ of sizes $n,n,m$ from $P_X,P_Y,P_Z \in \cal P_\sf{S}(\beta, d, C)$ for some constant $C$ with densities $p_X,p_Y,p_Y$ satisfying $\|p_X-p_Y\|_2\geq\epsilon$. 

Note that the heat-semigroup is an $L^2$-contraction ($\|\lambda\|_\infty\leq1$) and that $$\|\lambda\|_2^2 = \int G_\sigma(x,y)^2 \D\mu(x)\D\mu(y)\asymp\sigma^{-d}$$ up to constants depending on the dimension. \Cref{thm: MMDUpper} combined with \Cref{lem:gauss kernel approx} and a choice $\sigma\asymp\epsilon^{1/\beta}$ yields the following result.
\begin{proposition}
    Testing \eqref{eqn:lfht mix def} over the class $\cal P_\sf{S}$ with total error $\alpha$ is possible provided
\begin{align*}
    \min\{m,n\}&\gtrsim \frac{C \|\lambda\|_\infty \log(1/\alpha)(1+R)^2}{\delta^2\epsilon^2} \asymp \frac{\log(1/\alpha)(1+R)^2}{\delta^2\epsilon^2} \\
    \min\{n, \sqrt{nm}\} &\gtrsim \frac{C \|\lambda\|_2 \sqrt{\log(1/\alpha)}}{\epsilon^2\delta}  \asymp \frac{\sqrt{\log(1/\alpha)}}{\epsilon^{(2\beta+d/2)/\beta}\delta}, 
\end{align*}
where $\epsilon$ is the lower bound on the $L^2$-distance between $P_X, P_Y$ and $R$ is defined as in the assumption $(iii)$ of \Cref{item:mlfht3}.
\end{proposition}
Taking $R\asymp\delta\asymp\alpha=\Theta(1)$ above, we obtain new results for $\LF$ and using the reduction from two-sample testing given in \Cref{ssec:connections} we partly recover \cite[Theorem 5]{li2019optimality}. Only partly, because the above requires bounded density with respect to our base measure $\D\mu(x)=\exp(-x^2)\D x$. 

\textbf{Sample Complexity Lower Bounds}\quad 
Note that our lower bound \Cref{thm:lower bds} doesn't apply because the top eigenfunction of the Gaussian kernel is not constant. Once again, a more careful choice of the base measure (or kernel) might lead to a more suitable argument for the lower bound. We leave such pursuit as open.

\section{Black-box Boosting of Success Probability}\label{sec:boosting}

In this section we briefly describe how upper bounds on the minimax sample complexity in the constant error probability regime ($\alpha=\Theta(1)$) can be used to obtain the dependence $\log(1/\alpha)$ in the small error probability regime ($\alpha=o(1)$). We will argue abstractly in a way that applies to the setting of \Cref{thm: MMDUpper}.

Suppose that from some distributions $P_1,P_2,\dots,P_k$ we take samples $X^1,X^2,\dots,X^k$ of size $n_1,n_2,\dots,n_k$ respectively and are able to decide between two hypotheses $H_0$ and $H_1$ (fixed but arbitrary) with total error probability at most $1/3$. Call this test as $\Psi(X^1,\dots,X^k) \in \{0,1\}$, so that 
\begin{equation*}
    \P(\Psi(X^1,\dots,X^k) = 0 | H_0) \geq 2/3 \qquad\text{and}\qquad \P(\Psi(X^1,\dots,X^k)=1|H_1) \geq 2/3. 
\end{equation*}
Now, to each an error of $o(1)$, instead, we take $18n_1\log(2/\alpha), \dots, 18 n_k\log(2/\alpha)$ observations from $P_1$ through $P_k$, and split each sample into $18\log(2/\alpha)$ equal sized batches $\{X^{i,j}\}_{i \in [k], j \in [18\log(2/\alpha)]}$. Here $18\log(2/\alpha)$ is assumed to be an integer without loss of generality. 
The split samples form $18\log(2/\alpha)$ independent binary random variables $$A_j \eqdef  \Psi(X^{1,j}, \dots, X^{k,j})$$ for $j=1,2,\dots,18\log(2/\alpha)$. We claim that the majority voting test
\begin{equation*}
    \Psi_\alpha(\{X^{i,j}\}_{i,j}) = \begin{cases} 1 &\text{if } \bar A \geq 1/2 \\ 0 &\text{otherwise}\end{cases}
\end{equation*}
tests $H_0$ against $H_1$ with total probability of error at most $\alpha$, where $$\bar A \eqdef \frac{1}{18\log(2/\alpha)} \sum_{j=1}^{18\log(2/\alpha)} A_j.$$ Indeed, by Hoeffding's inequality, we have
\begin{align*}
    \P\left(\left.\bar A \geq 1/2 \right| H_0\right) &\leq \alpha/2 \\
    \P\left(\left.\bar A \leq 1/2 \right| H_1\right) &\leq \alpha/2.
\end{align*}
Therefore, in the remainder of our upper bound proofs, we only focus on achieving a constant probability of error ($\alpha = \Theta(1)$) as the logarithmic dependence follows by the above. 
\begin{remark}
    As mentioned in the discussion succeeding \Cref{cor:lower bd}, we do conjecture the \emph{tight} dependence in the upper bound to be $\sqrt{\log(\alpha^{-1})}$ instead of $\log(\alpha^{-1})$ shown by this method.
\end{remark}

\section{\texorpdfstring{ Proof of \Cref{thm: MMDUpper} }{ Proof of Theorem 3.2} }

\subsection{Notation and Technical Tools}
We use the expansion $$K(x,y) = \sum_\ell \lambda_\ell e_\ell(x) e_\ell(y)$$ extensively, where $\lambda \eqdef (\lambda_1, \lambda_2, \dots)$ are $K$'s eigenvalues (regarded as an integral operator on $L^2(\mu)$) in non-increasing order and $e_1,e_2,\dots$ are the corresponding eigenfunctions forming an orthonormal basis for $L^2(\mu)$, and convergence is to be understood in $L^2(\mu)$. We use the notation $\langle\,\cdot\,\rangle\eqdef\int\cdot\,\D\mu$.
For all $u \in L^2(\mu)$ we define 
$$u_\ell \eqdef \langle u e_\ell\rangle,\quad u_{\ell\ellp} \eqdef \langle u e_\ell e_\ellp\rangle,\quad \ell=1,2,\dots$$
and consequently $u=\sum_\ell u_\ell e_\ell$.
We also define
$$ K[u](\cdot) \eqdef \int K(t,\cdot)u(t)\mu(\D t) = \sum_\ell \lambda_\ell u_\ell e_\ell(\cdot),$$
where the second equality follows from the orthonormality of $\{e_\ell\}_{\ell=1}^\infty.$ 
Note that the RKHS embedding satisfies $\theta_u \eqdef \int K(x,\cdot) u(x)\D \mu(x) = K[u]$. 
Now, for $P_X$ we write 
$$x_\ell \eqdef (p_X)_\ell = \langle p_X e_\ell\rangle,\quad x_{\ell\ellp} \eqdef (p_X)_{\ell\ellp} = \langle p_X e_\ell e_\ellp\rangle,\quad \ell,\ellp=1,2,\dots$$
where $p_X$ is the $\mu$-density of $P_X$. The similar notations also apply to $P_Y,P_Z$. 
The following identities will be very useful in our proofs. 
\begin{lemma}\label{lem:calc rules}
    For each identity below, let $f,g,h \in L^2(\mu)$ be such that the quantity is well defined. Then, 
\begin{align}
    \|\theta_f\|_{\cal H_K}^2 &= \sum_\ell\lambda_\ell f_\ell^2\label{eqn:calc rule rkhs norm} \\
   \mmd^2(f,g) &= \sum_\ell \lambda_\ell (f_\ell-g_\ell)^2\label{eqn:calc rule mmd} \\
    \|K[f]\|_2^2 &= \sum_\ell \lambda_\ell^2 f_\ell^2 \label{eqn:calc rule |K[f]|}\\
    \sum_\ell \lambda_\ell f_\ell g_\ell &= \langle f K[g]\rangle = \langle K[f]g\rangle\label{eqn:calc rule K[f]g} \\
    \sum_{\ell\ellp} \lambda_\ell\lambda_\ellp h_{\ell\ellp} f_\ell g_\ellp &= \langle h K[f] K[g]\rangle \label{eqn:calc rule hK[g]K[f]} \\
    \sum_{\ell\ellp} \lambda_\ell\lambda_\ellp g_{\ell\ellp} f_{\ell\ellp} &= \sum_\ell\lambda_\ell\langle fe_\ell K[ge_\ell]\rangle.\label{eqn:calc rule ll}
\end{align}
Suppose that $f,g$ are probability densities with respect to $\mu$ that are bounded by $C$. Then 
\begin{equation}
    0 \leq \sum_{\ell\ellp} \lambda_\ell\lambda_\ellp g_{\ell\ellp} f_{\ell\ellp} \leq C^2 \|\lambda\|_2^2. 
\end{equation}
\end{lemma}

\begin{proof}
We prove each claim, starting with \eqref{eqn:calc rule rkhs norm}. Clearly
\begin{align*}
    \|\theta_f\|^2_{\cal H_K} &= \|K[f]\|^2_{\cal H_K} \\
    &= \left\|\int K(x,\cdot) f(x)\D\mu(x)\right\|^2_{\cal H_K} \\
    &= \iint \langle K(x,\cdot), K(y,\cdot)\rangle_{\cal H_K} f(x)f(y)\D\mu(x)\D\mu(y) \\
    &= \iint K(x,y)f(x)f(y)\D\mu(x)\D\mu(y) \\
    &= \sum_\ell\lambda_\ell f_\ell^2
\end{align*}
as required. The second claim \eqref{eqn:calc rule mmd} follows immediately from \eqref{eqn:calc rule rkhs norm} by definition. For \eqref{eqn:calc rule |K[f]|} by orthogonality we have
\begin{align*}
    \|K[f]\|_2^2 &= \|\sum_\ell \lambda_\ell f_\ell e_\ell\|_2^2 \\
    &= \sum_\ell \lambda_\ell^2 f_\ell^2. 
\end{align*}
For \eqref{eqn:calc rule K[f]g} by the definition of $K[\cdot]$ we have
\begin{align*}
    \sum_\ell \lambda_\ell f_\ell g_\ell &= \left\langle \left(\sum_\ell \lambda_\ell f_\ell e_\ell\right) g\right\rangle \\
    &= \langle K[f]g\rangle. 
\end{align*}
For \eqref{eqn:calc rule hK[g]K[f]} we can write
\begin{align*}
    \sum_{\ell\ellp} \lambda_\ell \lambda_\ellp h_{\ell\ellp}f_\ell g_\ellp &= \sum_\ell \lambda_\ell f_\ell \left\langle \left(\sum_\ellp \lambda_\ellp g_\ellp e_\ellp\right) he_\ell \right\rangle \\
    &= \sum_\ell \lambda_\ell f_\ell \langle K[g] he_\ell\rangle \\
    &= \langle K[g]hK[f]\rangle. 
\end{align*}
Finally, for \eqref{eqn:calc rule ll} we have
\begin{align*}
    \sum_{\ell\ellp} \lambda_\ell \lambda_\ellp f_{\ell\ellp} g_{\ell\ellp} &= \sum_\ell\lambda_\ell \left\langle\left(\sum_\ellp \lambda_\ellp  g_{\ell\ellp} e_\ellp\right) fe_\ell\right\rangle \\
    &= \sum_\ell\lambda_\ell \langle K[ge_\ell] fe_\ell\rangle. 
\end{align*}
Suppose now that $f,g$ are probability densities with respect to $\mu$ that are bounded by $C>0$. Let $X,Y$ be independent random variables following the densities $f,g$. Then 
\begin{align*}
    \sum_{\ell\ellp} \lambda_\ell\lambda_\ellp f_{\ell\ellp}g_{\ell\ellp} &= \E\left[\left(\sum_\ell \lambda_\ell e_\ell(X)e_\ell(Y)\right)^2\right] \\&\leq C^2 \int_\cal{X}\int_\cal{X} \left(\sum_\ell \lambda_\ell e_\ell(x)e_\ell(y)\right)^2 \D\mu(x)\D\mu(y) \\&= C^2 \|\lambda\|_2^2
\end{align*}
as claimed, where we used that the $e_\ell$ are orthonormal. 
\end{proof}

\subsection{Mean and Variance Computation}\label{ssec:mean var comp}
We take $\pi=\delta/2$. Our statistic reads
\begin{align*}
    -T(X,Y,Z)+\gamma(X,Y,\pi) &= \langle \theta_{\widehat P_Z} - (\bar\pi \theta_{\widehat P_X} + \pi \theta_{\widehat P_Y}), \theta_{\widehat P_X} - \theta_{\widehat P_Y}\rangle_{u,\cal H_K} 
    \\[3ex]
    &= \frac{1}{nm}\underbrace{\sum_{ij}k(X_i,Z_j)}_{\sf{I}} - \frac{1}{nm} \underbrace{\sum_{ij}k(Y_i,Z_j)}_{\sf{II}} - \frac{2\bar\pi}{n(n-1)}\underbrace{\sum_{i < i'}k(X_i,X_{i'})}_{\sf{III}} 
    \\[1ex]
        &\qquad + \frac{2\pi}{n(n-1)}\underbrace{\sum_{i<i'} k(Y_i,Y_{i'})}_{\sf{IV}} + \frac{\bar\pi-\pi}{n^2}\underbrace{\sum_{ij} k(X_i,Y_j)}_{\sf{V}}. 
\end{align*}
Recall that $\nu=\argmin_{\nu' \in \R} \mmd(P_Z,\bar\nu'P_X+\nu'P_Y)$. Let us write $z=\bar\nu x + \nu y + r$ for $1-\bar\nu=\nu$, where the residual term is denoted as $r \in L^2(\mu)$. 
Let $\theta_r = \int r(t) K(t,\cdot) \mu(\D t)$ be the mean embedding of $r$. Under both hypotheses we assume that $\|\theta_r\|_{\cal H_K} \leq R\cdot \mmd(P_X,P_Y)$, moreover $\langle \theta_r, \theta_{P_Y}-\theta_{P_X}\rangle_{\cal H_K} = 0$ by the definition of $\nu$. We look at each of the $5+\binom{5}{2}=15$ terms of the variance separately. 
\begin{align*}
    \var(\sf{I}) &= \sum_{\ell\ellp} \lambda_\ell\lambda_\ellp \Big\{n(n-1)m  (z_{\ell\ellp}-z_\ell z_\ellp) x_\ell x_\ellp + nm(m-1)(x_{\ell\ellp}-x_\ell x_\ellp)z_\ell z_\ellp \\ &\qquad\qquad\qquad\qquad+ nm (x_{\ell\ellp}z_{\ell\ellp}-x_\ell x_\ellp z_\ell z_\ellp)\Big\} 
    \\[2ex]
    \var(\sf{II}) &= \sum_{\ell\ellp} \lambda_\ell\lambda_\ellp \Big\{n(n-1)m  (z_{\ell\ellp}-z_\ell z_\ellp) y_\ell y_\ellp + nm(m-1)(y_{\ell\ellp}-y_\ell y_\ellp)z_\ell z_\ellp \\ &\qquad\qquad\qquad\qquad + nm (y_{\ell\ellp}z_{\ell\ellp}-y_\ell y_\ellp z_\ell z_\ellp)\Big\} 
    \\[2ex]
    \var(\sf{III}) &= \sum_{\ell\ellp}\lambda_\ell\lambda_\ellp \Big\{\binom{n}{2} (x_{\ell\ellp}^2-x_\ell^2x_\ellp^2)+(\binom{n}{2}^2-\binom{n}{2}-\binom{4}{2}\binom{n}{4})(x_{\ell\ellp}-x_\ell x_\ellp)x_\ell x_\ellp\Big\} 
    \\[2ex]
    \var(\sf{IV}) &= \sum_{\ell\ellp}\lambda_\ell\lambda_\ellp \Big\{\binom{n}{2} (y_{\ell\ellp}^2-y_\ell^2y_\ellp^2)+(\binom{n}{2}^2-\binom{n}{2}-\binom{4}{2}\binom{n}{4})(y_{\ell\ellp}-y_\ell y_\ellp)y_\ell y_\ellp\Big\} 
    \\[3ex]
    \var(\sf{V}) &= \sum_{\ell\ellp} \lambda_\ell\lambda_\ellp \Big\{n^2(n-1)  (y_{\ell\ellp}-y_\ell y_\ellp) x_\ell x_\ellp + n^2(n-1)(x_{\ell\ellp}-x_\ell x_\ellp)y_\ell y_\ellp \\&\qquad\qquad\qquad\qquad+ n^2 (x_{\ell\ellp}y_{\ell\ellp}-x_\ell x_\ellp y_\ell y_\ellp)\Big\}
\end{align*}
For the cross terms we obtain
\begin{align*}
    \cov(\sf{I}, \sf{II}) &= \sum_{\ell\ell'}\lambda_\ell\lambda_\ellp n^2m( z_{\ell\ellp}  - z_\ell z_\ellp) x_\ell y_\ellp 
    \\
    \cov(\sf{I}, \sf{III}) &= \sum_{\ell\ellp} \lambda_\ell\lambda_\ellp n(n-1)m(x_{\ell\ellp} - x_\ell x_\ellp)z_\ell x_\ellp 
    \\
    \cov(\sf{I},\sf{IV}) &= 0 
    \\[1ex]
    \cov(\sf{I},\sf{V}) &= \sum_{\ell\ellp}\lambda_\ell\lambda_\ellp n^2m (x_{\ell\ellp} -x_\ell x_\ellp)z_\ell y_\ellp 
    \\
    \cov(\sf{II}, \sf{III}) &= 0 
    \\[1ex]
    \cov(\sf{II}, \sf{IV}) &= \sum_{\ell\ellp} \lambda_\ell\lambda_\ellp n(n-1)m(y_{\ell\ellp} - y_\ell y_\ellp)z_\ell y_\ellp 
    \\
    \cov(\sf{II}, \sf{V}) &= \sum_{\ell\ellp}\lambda_\ell\lambda_\ellp n^2m(y_{\ell\ellp}-y_\ell y_\ellp)z_\ell x_\ellp 
    \\
    \cov(\sf{III}, \sf{IV}) &= 0 
    \\[1ex]
    \cov(\sf{III}, \sf{V}) &= \sum_{\ell\ellp}\lambda_\ell\lambda_\ellp n^2(n-1) (x_{\ell\ellp}-x_\ell x_\ellp)x_\ell y_\ellp 
    \\
    \cov(\sf{IV}, \sf{V}) &= \sum_{\ell\ellp}\lambda_\ell\lambda_\ellp n^2(n-1)(y_{\ell\ellp}-y_\ell y_\ellp)y_\ell x_\ellp.
\end{align*}
Note that $\binom{n}{2}^2-\binom{n}{2}-\binom{n}{2}\binom{n}{4} = n(n-1)^2-n(n-1)$. Collecting terms, and simplifying, we get the coefficient of the $\frac{1}{n}$ term: 
\begin{align*}
    \coef\left(\frac1n\right) 
    &= \sum_{\ell,\ellp} \lambda_\ell\lambda_\ellp\Bigg(\underbrace{(x_{\ell\ellp}-x_\ell x_\ellp)z_\ell z_\ellp}_{\var(\sf{I})} + \underbrace{(y_{\ell\ellp}-y_\ell y_\ellp)z_\ell z_\ellp}_{\var(\sf{II})} + \underbrace{4\bar\pi^2(x_{\ell\ellp}-x_\ell x_\ellp)x_\ell x_\ellp}_{\var(\sf{III})} 
    \\
    &\qquad + \underbrace{4\pi^2(y_{\ell\ellp}-y_\ell y_\ellp)y_\ell y_\ellp}_{\var(\sf{IV})} + \underbrace{(\bar\pi-\pi)^2(y_{\ell\ellp}-y_\ell y_\ellp)x_\ell x_\ellp + (\bar\pi-\pi)^2(x_{\ell\ellp}-x_\ell x_\ellp)y_\ell y_\ellp}_{\var(\sf{V})} 
    \\
    &\qquad- \underbrace{4\bar\pi(x_{\ell\ellp}-x_\ell x_\ellp)z_\ell x_\ellp}_{\cov(\sf{I}, \sf{III})} 
    + \underbrace{2(\bar\pi-\pi)(x_{\ell\ellp}-x_\ell x_\ellp)z_\ell y_\ellp}_{\cov(\sf{I},\sf{V})}
    \\
    &\qquad - \underbrace{4\pi(y_{\ell\ellp}-y_\ell y_\ellp)z_\ell y_\ellp}_{\cov(\sf{II},\sf{IV})} - \underbrace{2(\bar\pi-\pi)(y_{\ell\ellp}-y_\ell y_\ellp)z_\ell x_\ellp}_{\cov(\sf{II},\sf{V})} \\
    &\qquad- \underbrace{4\bar\pi(\bar\pi - \pi)(x_{\ell\ellp} -x_\ell x_\ellp)x_\ell y_\ellp}_{\cov(\sf{III},\sf{V})} + \underbrace{4\pi(\bar\pi-\pi)(y_{\ell\ellp}-y_\ell y_\ellp)y_\ell x_\ellp}_{\cov(\sf{IV}, \sf{V})}\Bigg).
\end{align*}
After expanding $z_\ell$ as $z_\ell=\bar\nu x_\ell + \nu y_\ell + r_\ell$, we split the calculation into multiple parts to simplify it. First, we focus on terms that are multiplied by $(x_{\ell\ellp}-x_\ell x_\ellp)$ and do not contain $r_\ell$ or $r_\ellp$. Using \Cref{lem:calc rules} extensively and the fact that $\bar\pi=1-\pi, \bar\nu=1-\nu$, we find that the sum of these terms equals
\begin{align*}
    & \bar\nu^2\langle x K[x]^2\rangle + \nu^2 \langle x K[y]^2\rangle + 2\bar\nu\nu \langle xK[x]K[y]\rangle
    - \bar\nu^2\langle x K[x]\rangle^2-\nu^2\langle x K[y]\rangle^2 - 2\bar\nu\nu\langle x K[x]\rangle\langle x K[y]\rangle 
    \\
    &+ 4\bar\pi^2 \langle x K[x]^2\rangle - 4\bar\pi^2 \langle xK[x]\rangle^2 
        + (\bar\pi-\pi)^2\langle xK[y]^2\rangle - (\bar\pi-\pi)^2\langle xK[y]\rangle^2 
    \\
    &- 4\bar\pi\bar\nu\langle x K[x]^2\rangle -4\bar\pi \nu \langle xK[x]K[y]\rangle 
        + 4\bar\pi\bar\nu\langle xK[x]\rangle^2+4\bar\pi\nu\langle xK[x]\rangle\langle xK[y]\rangle 
    \\
    &+ 2(\bar\pi-\pi)\bar\nu\langle xK[x]K[y]\rangle + 2(\bar\pi-\pi)\nu\langle xK[y]^2\rangle 
    - 2(\bar\pi-\pi)\bar\nu\langle xK[x]\rangle\langle xK[y]\rangle - 2(\bar\pi-\pi)\nu\langle xK[y]\rangle^2 
    \\
    &- 4\bar\pi(\bar\pi-\pi)\langle xK[x]K[y]\rangle + 4\bar\pi(\bar\pi-\pi)\langle x K[x]\rangle\langle xK[y]\rangle 
    \\
    =& (\bar\nu-2\bar\pi)^2 \Big(\langle xK[x-y]^2\rangle -\langle xK[x-y]\rangle^2\Big) 
    \\
    \leq& \, C \, \|\lambda\|_\infty \mmd^2(P_X,P_Y). 
\end{align*}
Similarly, the terms involving $(y_{\ell\ellp}-y_\ell y_\ellp)$ but not $r_\ell$ or $r_\ellp$ sum up to the quantity
\begin{align*}
    (\nu-2\pi)^2\Big(\langle yK[x-y]^2\rangle - \langle yK[x-y]\rangle^2\Big) &\leq C\|\lambda\|_\infty \mmd^2(P_X,P_Y). 
\end{align*}
Next, collecting the terms involving both $(x_{\ell\ellp}-x_\ell x_\ellp)$ and $r_\ell$ or $r_\ellp$ we get 
\begin{align*}
    &2\bar\nu\langle x K[r]K[x]\rangle +2 \nu\langle x K[r]K[y]\rangle + \langle xK[r]^2\rangle - 2\bar\nu\langle xK[x]\rangle\langle xK[r]\rangle - 2\nu\langle xK[y]\rangle\langle xK[r]\rangle - \langle xK[r]\rangle^2 
    \\
    &-4\bar\pi\langle xK[x]K[r]\rangle + 4\bar\pi\langle xK[x]\rangle\langle xK[r]\rangle 
    \\
    &+ 2(\bar\pi-\pi)\langle xK[y]K[r]\rangle - 2(\bar\pi-\pi)\langle xK[y]\rangle\langle xK[r]\rangle 
    \\
    =&\, 2(\bar\nu-2\bar\pi)\Big(\langle xK[r]K[x-y]\rangle - \langle xK[r]\rangle\langle xK[x-y]\rangle\Big) + \langle xK[r]^2\rangle-\langle xK[r]\rangle^2 
    \\
    \lesssim&\, C\, \|\lambda\|_\infty (R+R^2) \mmd^2(P_X,P_Y). 
\end{align*}
Finally, collecting the terms involving both $(y_{\ell\ellp}-y_\ell y_\ellp)$ and $r_\ell$ or $r_\ellp$ we get
\begin{align*}
    & 2(\nu-2\pi)\Big(\langle yK[r]K[y-x]\rangle - \langle yK[r]\rangle\langle yK[y-x]\rangle\Big) + \langle yK[r]^2\rangle-\langle yK[r]\rangle^2 \\
    \lesssim&\, C \|\lambda\|_\infty (R+R^2) \mmd^2(P_X,P_Y). 
\end{align*}

Similarly we get
\begin{align*}
    \coef\left(\frac1m\right) &= \sum_{\ell\ellp} \lambda_\ell\lambda_\ellp\Bigg(\underbrace{(z_{\ell\ellp}-z_\ell z_\ellp)x_\ell x_\ellp}_{\var(\sf{I})} + \underbrace{(z_{\ell\ellp}-z_\ell z_\ellp)y_\ell y_\ellp}_{\var(\sf{I})} + \underbrace{2(z_{\ell\ellp}-z_\ell z_\ellp)x_\ell y_\ellp}_{\cov(\sf{I}, \sf{II})}\Bigg) 
    \\
    &= \langle zK[x-y]^2\rangle-\langle zK[x-y]\rangle^2 
    \\[1ex]
    &\lesssim C\|\lambda\|_\infty \mmd^2(P_X,P_Y).
\end{align*}
The remaining coefficients don't rely on subtle cancellations, and simple bounds yield
\begin{align*}
    \coef\left(\frac{1}{n(n-1)}\right) &= \sum_{\ell\ellp} \lambda_\ell\lambda_\ellp\Bigg( \underbrace{4\bar\pi^2\left(\frac12(x_{\ell\ellp}^2-x_\ell^2 x_\ellp^2)-(x_{\ell\ellp}-x_\ell x_\ellp)x_\ell x_\ellp\right)}_{\var(\sf{III})}
    \\
    &\qquad\qquad+ \underbrace{4\pi^2\left(\frac12(y_{\ell\ellp}^2-y_\ell^2 y_\ellp^2)-(y_{\ell\ellp}-y_\ell y_\ellp)y_\ell y_\ellp\right)}_{\var(\sf{IV})} \Bigg)
    \\
    &\lesssim C^2 \|\lambda\|_2^2
    \\[2ex]
    \coef\left(\frac{1}{nm}\right) 
    &= \sum_{\ell\ellp} \lambda_\ell\lambda_\ellp\Bigg(\underbrace{-(z_{\ell\ellp}-z_\ell z_\ellp)x_\ell x_\ellp - (x_{\ell\ellp}-x_\ell x_\ellp)z_\ell z_\ellp + (x_{\ell\ellp}z_{\ell\ellp}-x_\ell x_\ellp z_\ell z_\ellp)}_{\var(\sf{I})} 
    \\
    &\qquad - \underbrace{(z_{\ell\ellp}-z_\ell z_\ellp)y_\ell y_\ellp - (y_{\ell\ellp}-y_\ell y_\ellp)z_\ell z_\ellp + (y_{\ell\ellp}z_{\ell\ellp}-y_\ell y_\ellp z_\ell z_\ellp)}_{\var(\sf{I})} \Bigg)
    \\
    &\lesssim C^2 \|\lambda\|_2^2 
    \\[2ex]
    \coef\left(\frac{1}{n^2}\right) 
    &= \sum_{\ell\ellp} \lambda_\ell\lambda_\ellp\Bigg( \underbrace{(\bar\pi-\pi)\left(-(y_{\ell\ellp}-y_\ell y_\ellp)x_\ell x_\ellp - (x_{\ell\ellp}-x_\ell x_\ellp)y_\ell y_\ellp + (x_{\ell\ellp}y_{\ell\ellp}-x_\ell x_\ellp y_\ell y_\ellp)\right)}_{\var(\sf{V})} \Bigg) 
    \\
    &\lesssim C^2 \|\lambda\|_2^2. 
\end{align*}

Summarizing, we've found that 
\begin{equation}\label{eq:var=1/n2+1/mn}
    \begin{aligned}
        \var(T(X,Y,Z)-\gamma(X,Y,\pi)) &\lesssim \left(\frac1n + \frac1m\right) C \|\lambda\|_\infty (1+R^2) \mmd^2(P_X,P_Y) 
        \\&\qquad
        + \left(\frac{1}{n^2} + \frac{1}{nm}\right) C^2 \|\lambda\|_2^2. 
    \end{aligned}
\end{equation}
Using that $\langle \theta_r, \theta_{P_Y}-\theta_{P_X}\rangle_{\cal H_K} = 0$, we compute the expectation to be 
\begin{equation*}
\E\left[-T(X,Y,Z)+\gamma(X,Y,\pi)\right] = (\pi-\nu)\mmd^2(P_X,P_Y).     
\end{equation*}
Taking $\pi \eqdef \delta/2$ and applying Chebyshev's inequality shows that there exists a universal constant $c>0$, such that the testing problem is possible at constant error probability (say $\alpha=5\%$), provided that the sample sizes $m,n$ satisfy the following inequalities:
\begin{align*}
    \min\{m,n\} &\geq c\frac{C\|\lambda\|_\infty(1+R^2)}{\delta^2\epsilon^2} \\
    \min\{n,\sqrt{nm}\} &\geq c\frac{C\|\lambda\|_2}{\delta\epsilon^2}. 
\end{align*}
By repeated sample splitting and majority voting (see \Cref{sec:boosting}), we can boost the success probability of this test to the desired level $1-\alpha$ by incurring a multiplicative $\Theta(\log(1/\alpha))$ factor on the sample sizes $n,m$, which yields the desired result.

\section{\texorpdfstring{ Proof of \Cref{thm:lower bds} }{ Proof of Theorem 3.3 } }

\subsection{Information theoretic tools}

Our lower bounds rely on the method of two fuzzy hypotheses \cite{tsybakov}. Given a measurable space $\cal S$, let $\cal M(\cal S)$ denote the set of all probability measures on $\cal S$. We call subsets $H \subseteq \cal M(\cal S)$ hypotheses. The following is the main technical result that our proofs rely on.

\begin{lemma}\label{lem:lower bound main}
Take hypotheses $H_0,H_1 \subseteq \cal M(\cal S)$ and $P_0,P_1 \in \cal M(\cal S)$ random with $\P(P_i \in H_i) = 1$. Then
\begin{align*}
    \inf\limits_{\psi}\max_{i=0,1}\sup\limits_{P \in H_i} P(\psi\neq i) \geq \frac12\left(1-\TV(\bb E P_0, \bb E P_1)\right),
\end{align*}
where the infimum is over all tests $\psi:\cal X \to \{0,1\}$.
\end{lemma}
\begin{proof}
For any $\psi$
\begin{align*}
    \max_{i=0,1}\sup\limits_{\bb P_i\in H_i} \bb P_i(\psi\neq i) &\geq \frac12\sup\limits_{\bb P_i\in H_i} (\P_0(\psi=1)+\P_1(\psi=0)) \\
    &\geq\frac12\E\Big[P_0(\psi=1)+P_1(\psi=0)\Big].
\end{align*}
Optimizing over $\psi$ we get that the RHS above is equal to $\frac12(1-\TV(\E P_0, \E P_1))$ as required. 
\end{proof}
Therefore, to prove a lower bound on the minimax sample complexity of testing with total error probability $\alpha$, we just need to construct two random measures $P_i \in H_i$ such that $1-\TV(\E P_0, \E P_1) = \Omega(\alpha)$. In our proofs we also use the following standard results on $f$-divergences. 
\begin{lemma}[{{\cite[Section 7]{yuryyihongbook}}}] \label{lem:f-div ineqs}
    For any probability distributions $P,Q$ the inequalities
    \begin{align*}
        1-\TV(P,Q) \geq \frac12\exp(-\KL(P\|Q)) \geq \frac12\frac{1}{1+\chi^2(P\|Q)}
    \end{align*}
    hold. 
\end{lemma}

\begin{lemma}[Chain rule for $\chi^2$-divergence]\label{lem:chi2 chain}
    Let $P_{X,Y}, Q_{X,Y}$ be probability measures such that the marginals on $X$ are equal ($P_X=Q_X$). Then 
    \begin{equation*}
        \chi^2(P_{X,Y} \| Q_{X,Y}) = \chi^2(P_{Y|X} \| Q_{Y|X} | P_X). 
    \end{equation*}
\end{lemma}

\begin{proof}
Let $P_{X,Y}, Q_{X,Y}$ have densities $p,q$ with respect to some $\mu$. Then, by some abuse of notation, we have
    \begin{align*}
        \chi^2(P_{X,Y} \| Q_{X,Y}) &= -1 + \int \frac{p(x,y)^2}{q(x,y)} \D\mu(x,y) \\
        &= -1 + \int \frac{p(y|x)^2 p(x)}{q(y|x)} \D\mu(x,y) \\
        &= \int p(x) \int \left(\frac{p(y|x)^2}{q(y|x)} -1 \right)\D\mu(y,x) \\
        &= \chi^2(P_{Y|X} \| Q_{Y|X} | P_X). 
    \end{align*}
\end{proof}

\subsection{Constructing hard instances}
Recall that in the statement of \Cref{thm:lower bds}, we assume that $\mu(\cal X)=1$, $\sup_{x\in\cal X}K(x,x)\leq1$ and $\int K(x,y) \mu(\D x) \equiv \lambda_1$. Let $f_0\equiv1$ and for each $\eta \in \{\pm1\}^\N$ define
\begin{equation}
    f_\eta = 1 + \epsilon \underbrace{\sum_{j\geq2} \rho_j \eta_j e_j}_{\eqdef g_\eta}
\end{equation}
where $\{\rho_j\}_{j\geq2}$ is chosen as $\rho_j = \one\{2\leq j\leq J\} \sqrt{\lambda_j}/\|\lambda\|_{2,J}$, where we define $\|\lambda\|_{2,J}=\sqrt{\sum_{2\leq j \leq J} \lambda_j^2}$ for some $J\geq2$. Notice that $\int f_\eta(x) \mu(\D x)=\mu(\cal X)=1$ due to orthogonality of the eigenfunctions. Assume from here on that $J$ is chosen so that for all $\eta$ we have $f_\eta(x) \geq 1/2$ for all $x \in \cal X$. This makes $f_\eta$ into a valid probability density with respect to the base measure $\mu$. Before continuing, we prove the following Lemma, which gives a lower bound on the maximal $J$ for which $f_\eta \geq 1/2$ for all $\eta$. 

\begin{lemma}\label{lem:J bound}
$J\leq J^\star_\epsilon$ holds provided $2\epsilon\sqrt{J-1} \leq \|\lambda\|_{2J}$.
\end{lemma}

\begin{proof}[Proof of \Cref{lem:J bound}]
Notice that
\begin{equation}
    \|e_j\|_\infty = \sup_{x\in\cal X}\langle K(x,\cdot),e_j\rangle_{\cal H} \leq \sup_{x \in \cal X} \|K(x,\cdot)\|_{\cal H} \|e_j\|_{\cal H} \leq \frac{1}{\sqrt \lambda _j}, 
\end{equation}
where we use $\|K(x,\cdot)\|_{\cal H} = \sqrt{K(x,x)}$. We have
\begin{align*}
   \|g_\eta\|_\infty &= \epsilon \|\sum_{j\geq2}\rho_j \eta_j e_j\|_\infty
   = \epsilon \sup\limits_{x\in\cal X} \langle K(x,\cdot), \sum_{j\geq2} \rho_j\eta_j e_j\rangle_{\cal H} \\
    &\leq \epsilon \|\sum_{j\geq2} \rho_j\eta_j e_j\|_{\cal H} = \epsilon \sqrt{\sum_{j\geq2} \rho_j^2/\lambda_j} = \frac{\epsilon \sqrt{J-1}}{\|\lambda\|_{2,J}}, 
\end{align*}
and the result follows. 
\end{proof}
Note that \Cref{lem:J bound} immediately gives us a proof of \Cref{cor:lower bd}.
\begin{proof}[Proof of \Cref{cor:lower bd}]
    Suppose that $J$ is such that $\sum_{j=2}^J \lambda_j^2 \geq c^2 \|\lambda\|_2^2$. Then, by \Cref{lem:J bound}, if $\epsilon \leq \|\lambda\|_{2J}/(2\sqrt{J-1})$ then $J\leq J_\epsilon^\star$. By assumption, this is implied by the inequality $\epsilon \leq c\|\lambda\|_2/(2\sqrt{J-1})$, and the result follows. 
\end{proof}

Continuing with our proof, note that by construction we have 
\begin{equation}
    \mmd^2(f_0,f_\eta) = \sum_{j\geq2} \lambda_j\rho_j^2 = \epsilon^2, \quad \forall \eta\in\{\pm1\}^\N. 
\end{equation}

\subsubsection{Lower Bound on \texorpdfstring{$m$}{m}}


Again, we apply \Cref{lem:lower bound main} with the new (deterministic) construction
\begin{equation}
    P_0 = f_0^{\otimes n} \otimes (1+\eps e_2/\sqrt{\lambda_2})^n \otimes (1+\delta\eps e_2/\sqrt{\lambda_2})^{\otimes m}, 
    \qquad 
    P_1 = f_0^{\otimes n} \otimes (1+\eps e_2/\sqrt{\lambda_2})^n \otimes f_0^{\otimes m},
\end{equation}
where we write $f_\one = f_{(1,1,\dots)}$ and similarly for $g_\one$. By the data-processing inequality for $\chi^2$-divergence (also by \Cref{lem:chi2 chain}), we may drop the first $2n$ coordinates and obtain
\begin{align*}
    \chi^2(\E P_0, \E P_1) 
    &= \chi^2((1+\delta\epsilon e_2/\sqrt{\lambda_2})^{\otimes m} \| f_0^{\otimes m}) \\
    &= \left( 1+\delta^2\epsilon^2/{\lambda_2}  \right)^m-1 \\
    &\leq \exp\left(  \frac{\delta^2\epsilon^2m}{\lambda_2} \right)-1. 
\end{align*}
By Lemma \ref{lem:f-div ineqs} we 
\begin{equation*}
    1-\TV(\E P_0, \E P_1) \gtrsim \frac{1}{\chi^2(\E P_0, \E P_1)-1} \geq \exp(-\delta^2\epsilon^2 m) \stackrel{!}{=} \Omega(\alpha). 
\end{equation*}
The lower bound $m \gtrsim \lambda_2\log(1/\alpha)/(\delta\epsilon)^2$ now follows readily. 

\subsubsection{Lower Bound on \texorpdfstring{$n$}{n}}
Once again, we apply \Cref{lem:lower bound main} to the new construction
\begin{equation}
    P_0 = f_0^{\otimes n} \otimes f_\eta^{\otimes n} \otimes f_0^{\otimes m}, \qquad P_1 = f_\eta^{\otimes n} \otimes f_0^{\otimes n} \otimes f_0^{\otimes m},
\end{equation}
where we put a uniform prior on $\eta \in \{\pm1\}^\N$ as before. Using the subadditivity of total variation under products, we compute
\begin{align*}
    \TV(\E P_0, \E P_1) &= \TV(f_0^{\otimes n} \otimes \E f_\eta^{\otimes n}, \E[f_\eta^{\otimes n}] \otimes f_0^{\otimes n}) \\
    &\leq 2 \TV(\E f_\eta^{\otimes n}, f_0^{\otimes n}).
\end{align*}
Just as in \Cref{sec:mn lower} we upper bound by the $\chi^2$-divergence to get
\begin{align*}
    \chi^2(\E f_\eta^{\otimes n} \| f_0^{\otimes n}) &= -1 + \E_{\eta\eta'} \int \prod\limits_{i=1}^n (f_\eta(x_i)f_{\eta'}(x_i)) \mu(\D x_1) \dots \mu(\D x_n) \\
    &\leq -1 + \E \exp(n\epsilon^2\sum_{j\geq2}\rho_j^2\eta_j\eta'_j) \\
    &= -1 + \prod\limits_{j\geq2} \cosh(n\epsilon^2\rho_j^2) \\
    &\leq -1 + \exp(n^2\epsilon^4 \sum_{j\geq2} \rho_j^4) \\
    &= -1 + \exp(n^2\epsilon^4/\|\lambda\|_{2,J}^2). 
\end{align*}
Again, by \Cref{lem:f-div ineqs} we obtain
\begin{equation*}
    1-\TV(\E P_0, \E P_1) \gtrsim \frac{1}{\chi^2(\E P_0 \| \E P_1)-1} \geq \exp(-n^2\epsilon^4/\|\lambda\|_{2,J}^2) \stackrel{!}{=} \Omega(\alpha). 
\end{equation*}
The lower bound $n\gtrsim\sqrt{\log(1/\alpha)}\|\lambda\|_{2,J}/\epsilon^2$ now follows readily.

\subsubsection{Lower Bound on \texorpdfstring{$m \cdot n$}{m*n}}
\label{sec:mn lower}
We take a uniform prior on $\eta$ and consider the random measures
\begin{equation}
    P_0 = f_0^{\otimes n} \otimes f_\eta^{\otimes n} \otimes ((1-\delta)f_0+\delta f_\eta)^{\otimes m} \qquad\text{and}\qquad P_1 = f_0^{\otimes n} \otimes f_\eta^{\otimes n} \otimes f_0^{\otimes m}. 
\end{equation}
Our goal is to apply \Cref{lem:lower bound main} to $P_0,P_1$. Notice that $(1-\delta)f_0+\delta f_\eta = 1 + \delta\epsilon g_\eta$. Let us write $X,Y,Z$ for the marginals first $n$, second $n$ and last $m$ coordinates of $P_0$ and $P_1$. By the data processing inequality and the chain rule \Cref{lem:chi2 chain} we have
\begin{align*}
    \chi^2(\E P_0 \| \E P_1) &= \chi^2((\E P_0)_{Y,Z} \| (\E P_1)_{Y,Z}) \\
    &= \chi^2((\E P_0)_{Z|Y} \| (\E P_1)_{Z|Y} | (\E P_0)_Y) \\
    &= \E \chi^2\left(\E\left[\left.(1+\delta\epsilon g_\eta)^{\otimes m}\right| Y\right] \| f_0^{\otimes m}\right) =: (\dagger).
\end{align*}
Notice that the expectation inside the $\chi^2$-divergence is with respect to $\eta$ given the variables $Y$, or in other words, over the posterior of $\eta$ with uniform prior given $n$ observations from the density $1+\epsilon g_\eta=f_\eta$. The outer expectation is over $Y$. Given $Y$, let $\eta$ and $\eta'$ be i.i.d. from said posterior. We get the bound
\begin{align*}
(\dagger) +1 &\leq \E \int \prod\limits_{i=1}^m (1+\delta\epsilon g_\eta(x_i))(1+\delta\epsilon g_{\eta'}(x_i)) \mu(\D x_i) \\
&= \E (1+\delta^2\epsilon^2 \sum_{j\geq2} \rho_j^2 \eta_j\eta'_j)^m \\
&\leq \E \exp(\delta^2\epsilon^2 m \sum_{j\geq2} \rho_j^2\eta_j\eta'_j). 
\end{align*}
Define the collections of variables $\eta_{-j} = \{\eta_j\}_{j\geq2}\setminus\{\eta_j\}$ and $\eta'_{-j}$ similarly. We shall prove the following claim:
\begin{equation}\label{eqn:lower peeling claim}
\E\left[\left.\exp(\delta^2\epsilon^2 m \rho_j^2 \eta_j\eta'_j) \right| \eta_{-j}\eta'_{-j}\right] \leq \exp(c\delta^2\epsilon^4 (\delta^2 m^2+m n) \rho_j^4)
\end{equation}
for some universal constant $c>0$. Assuming that \eqref{eqn:lower peeling claim} holds, by induction we can show that 
\begin{align*}
    (\dagger)+1 &\leq \exp(c\delta^2(\delta^2m^2+mn)\epsilon^4\sum_{j\geq2}\rho_j^4) \\
    &= \exp(c\delta^2(\delta^2m^2+mn)\epsilon^4/\|\lambda\|_{2,J}^2).
\end{align*}
Thus, if $mn +\delta^2m^2= o\left(\|\lambda\|_{2,J}^2/(\delta^2\epsilon^4)\right)$ then testing is impossible. 

We now prove \eqref{eqn:lower peeling claim}. Since the variable $\eta'_j\eta'_j$ is either $1$ or $-1$, we have
\begin{align*}
\E\left[\left.\exp(\delta^2\epsilon^2 m \rho_j^2 \eta_j\eta'_j) \right| \eta_{-j}\eta'_{-j}\right] &= (e^{\delta^2 \epsilon^2 m \rho_j^2}-e^{-\delta^2 \epsilon^2 m \rho_j^2})\cdot \bb P(\eta_j\eta'_j=1|\eta_{-j}\eta'_{-j}) + e^{-\delta^2 \epsilon^2 m \rho_j^2}.
\end{align*}
Let us write $\eta_{\pm1,j}$ for the vector of signs equal to $\eta$ but whose $j$'th coordinate is $\pm1$ respectively. Looking at the probability above, and using the independence of $\eta,\eta'$ given $Y$, we have
\begin{align*}
    \bb P(\eta_j\eta'_j = 1 | Y,\eta_{-j},\eta'_{-j}) &= \bb P(\eta_j=1|Y,\eta_{-j})^2 + \bb P(\eta_j=-1|Y,\eta_{-j})^2 \\
    &= \frac14\frac{(f_{\eta_{1j}}^{\otimes n}(Y))^2 + (f^{\otimes n}_{\eta_{-1j}}(Y))^2}{\left(\frac12f^{\otimes n}_{\eta_{1j}}(Y) + \frac12f^{\otimes n}_{\eta_{-1j}}(Y)\right)^2}.
\end{align*}
Taking the expectation $\E[\,\cdot\, | \eta_{-j},\eta'_{-j}]$ and using the HM-AM inequality $(\frac12(x+y))^{-1} \leq \frac12(\frac1x+\frac1y)$ valid for all $x,y > 0$ gives
\begin{align*}
    \P(\eta_j\eta'_j=1 | \eta_{-j},\eta'_{-j}) &= \frac14 \int \frac{(\prod_{i=1}^nf_{\eta_{1j}}(x_i))^2 + (\prod_{i=1}^nf_{\eta_{-1j}}(x_i))^2}{\frac12 \prod_{i=1}^nf_{\eta_{1j}}(x_i)+\frac12 \prod_{i=1}^nf_{\eta_{-1j}}(x_i)} \mu(\D x_1) \dots \mu(\D x_n) \\
    &\leq \frac14 + \frac18 \int\left(\frac{(\prod_{i=1}^n f_{\eta_{1j}}(x_i))^2}{\prod_{i=1}^n f_{\eta_{-1j}}(x_i)} + \frac{(\prod_{i=1}^n f_{\eta_{-1j}}(x_i))^2}{\prod_{i=1}^n f_{\eta_{1j}}(x_i)}\right) \mu(\D x_1) \dots \mu(\D x_n) = (\star). 
\end{align*}
Note that $f_{\eta_{1j}} = f_{\eta_{-1j}} + 2\epsilon\rho_je_j$. Using the lower bound $f_{\eta_{\pm1j}}(x) \geq \frac12$ for all $x \in \cal X$ and the inequality $1+x\leq\exp(x)$, we get
\begin{align*}
    (\star) &\leq \frac14 + \frac18
    \left[\left(1 + \int\frac{4\epsilon^2\rho_j^2 e_j^2(x)}{f_{\eta_{-1j}}(x)} \mu(\D x)\right)^n + \left(1 + \int \frac{4\epsilon^2\rho_j^2 e_j^2(x)}{f_{\eta_{1j}}(x)} \mu(\D x)\right)^n\right] \\
    &\leq \frac14(1+e^{8\epsilon^2 n\rho_j^2}). 
\end{align*}
Recall that $(\star)$ is a probability so $(\star)\leq 1$, and we obtain
\begin{align*}
    (\star)\leq \frac14(1+e^{8\epsilon^2 n\rho_j^2 \wedge \ln3}). 
\end{align*}

Putting it together and applying \Cref{lem:local ineqs} we get
\begin{align*}
    \text{LHS of }\eqref{eqn:lower peeling claim} &\leq (e^{\delta^2\epsilon^2m\rho_j^2}-e^{-\delta^2\epsilon^2m\rho_j^2})\frac14(1+e^{8\epsilon^2n\rho_j^2 \wedge \ln3}) + e^{-\delta^2\epsilon^2m\rho_j^2} \\
    &\leq 
    e^{c\delta^2\epsilon^4\rho_j^4(\delta^2 m^2+mn)}
\end{align*}
for universal $c=16>0$. Thus, by \Cref{lem:f-div ineqs} we obtain
\begin{align*}
    1-\TV(\E P_0, \E P_1) \gtrsim \frac{1}{\chi^2(\E P_0, \E P_1)+1} \geq \exp(-c\delta^2\epsilon^4(\delta^2 m^2+mn)/\|\lambda\|_{2,J}^2) \stackrel{!}{=} \Omega(\alpha).
\end{align*}
The necessity of 
$$mn + \delta^2m^2 \gtrsim \frac{\log(1/\alpha) \|\lambda\|_{2,J}^2}{\delta^2\epsilon^4}$$
follows immediately.\footnote{We have $mn + m^2 \leq (\sqrt{mn} + m)^2 \leq 2(mn + m^2)$, so $\sqrt{mn} + m\asymp\sqrt{mn+m^2}$.}

\begin{lemma}\label{lem:local ineqs}
    For $a, b \geq 0$, the following inequality holds:
    \begin{equation*}
        \frac14(e^a-e^{-a})(1+e^{b \wedge \ln3})+e^{-a}\leq e^{2(ab+a^2)}.
\end{equation*}
\end{lemma}
\begin{proof}
    If $b\geq\ln 3$ or $a\geq 1$ we have: 
    \begin{align*}
        \text{LHS}\leq \frac14(e^a-e^{-a})(1+e^{\ln3})+e^{-a}
        = e^a
        \leq
        e^{\frac{b}{\ln3}a+a^2}.
    \end{align*}
    If $b<\ln 3$ and $a<1$, we have 
    $$e^b\leq 1+\frac{2}{\ln3}b\leq 1+2b,\quad \frac{e^a+e^{-a}}{2}\leq e^{a^2},\quad \frac{e^a-e^{-a}}{2}\leq \frac{e-e^{-1}}{2}a\leq 2a,$$ 
    and then
    \begin{align*}
        \frac14(e^a-e^{-a})(1+e^{b})+e^{-a}
        &= \frac{1}{2}(e^a+e^{-a})+\frac{e^b-1}{4}(e^a-e^{-a})\\&\leq e^{a^2}+2ab
        \\&\leq e^{a^2}(1+2ab)\\
        &\leq e^{a^2+2ab}
    \end{align*}
    The result follows from $\ln3>1$. 
\end{proof}

\section{Proofs From \Cref{sec:learning kernels}}

\subsection{Computing $\hat{\sigma}$}\label{sec:sigma defn}
We follow the implementation of $\widehat{\sigma}^2$ in \cite{liu2020learning}.
Given $X_1,\dots,X^{\sf{tr}}_{n_{\sf{tr}}}$ sampled from $P_X$ and $Y_1,\dots,Y^{\sf{tr}}_{n_{\sf{tr}}}$ sampled from $P_Y$, denote 
\begin{equation}
    H_{ij} :=
    K(X^{\sf{tr}}_i, X^{\sf{tr}}_j)
    + K(Y^{\sf{tr}}_i, Y^{\sf{tr}}_j)
    - K(X^{\sf{tr}}_i, Y^{\sf{tr}}_j)
    - K(Y^{\sf{tr}}_i, X^{\sf{tr}}_j),\quad i,j\in[n_{\sf{tr}}].
\end{equation}
Then $\widehat{\sigma}^2$ is computed via
\begin{equation} \label{eq:estimate_sigma_H1}
    \hat{\sigma}^2(X^{n_{\sf{tr}}},Y^{n_{\sf{tr}}};K)=\frac{4}{n_{\sf{tr}}^3} \sum_{i=1}^{n_{\sf{tr}}} \left( \sum_{j=1}^{n_{\sf{tr}}} H_{ij} \right)^2
    - \frac{4}{n_{\sf{tr}}^4}\left( \sum_{i=1}^{n_{\sf{tr}}} \sum_{j=1}^{n_{\sf{tr}}} H_{ij} \right)^2.
\end{equation}
Note that $\hat{\sigma}^2$ being non-negative follows from the AM-GM inequality.

\subsection{Heuristic Justification of the Objective \texorpdfstring{\eqref{eq:J}}{} }\label{sec:training obj justification}
As usual, let $X,Y,Z$ denotes samples of sizes $n,n,m$ from $P_X,P_Y,P_Z$ respectively. Let us give a heuristic justification for using the training objective defined in \eqref{eq:J} for the purpose of obtaining a kernel for $\LF/\mLF$. Note that originally it was proposed as a training objective for kernels to be used in two sample testing. Recall that our test for $\LF$ can be written as
\begin{align*}
    \Psi_{1/2}(X,Y,Z) 
    = \one\Big\{T_{\sf{LF}}
    \geq0\Big\} 
\end{align*}
where
\begin{align*}
T_{\sf{LF}}&=\mmd_u^2(\widehat{P}_Z,\widehat{P}_Y; K) - \mmd_u^2(\widehat{P}_Z,\widehat{P}_X; K),
\end{align*}
Heuristically, to maximize the power of (\ref{eqn:lfht mix def}), we would like to maximize the following population quantity
\begin{equation*}
    J_{\sf{LF}} \eqdef
    \frac{\E_0[T_{\sf{LF}}] - \E_1[T_{\sf{LF}}]}{\sqrt{\var_0(T_{\sf{LF}})}}
\end{equation*}
where 
\begin{align*}
    \E_0[T_{\sf{LF}}]&= \E_{X,Y,Z}[T_\sf{LF}|{P_Z=P_X}] = + \mmd^2(P_X, P_Y; K) ,
    \\
    \E_1[T_{\sf{LF}}]&= \E_{X,Y,Z}[T_\sf{LF}|{P_Z=P_Y}] = -\mmd^2(P_X, P_Y; K).
\end{align*}
Let $T_\sf{TS} = \mmd_u(\hat P_X, \hat P_Y)$ be the usual statistic that is thresholded for two-sample testing. Then, a computation analogous to that in Section \ref{ssec:mean var comp} show (cf. \eqref{eq:var=1/n2+1/mn}) that 
\begin{align*}
    \var_0(T_{\sf{LF}}) &\approx \frac{A(K,P_X,P_Y)}{n} +\frac{A(K,P_X,P_Y)}{m} +\frac{B(K,P_X,P_Y)}{n^2} +\frac{B(K,P_X,P_Y)}{mn},
    \\
    \var_0(T_{\sf{TS}}) &\approx \frac{A(K,P_X,P_Y)}{n} +\frac{B(K,P_X,P_Y)}{n^2}
\end{align*}
for some $A(K)$ and $B(K)$. Therefore, we have approximately
\begin{equation*}
    J_{\sf{LF}} \approx \frac{2 \mmd^2(P_X, P_Y; K)}{\sqrt{1+\frac{n}{m}}\sqrt{\var_0(T_{\sf{TS}})}} 
    \approx 
    2\sqrt{\frac{m}{m+n}}\widehat{J}(X,Y;K)
\end{equation*}
which only differs from our optimization objective defined in \eqref{eq:J} by a constant factor.

Second, notice that $\frac{\mmd(P_X, P_Y; K)}{\sqrt{\var(T_{\sf{TS}})}}$ depends only on $P_X-P_Y$ and that $((1-\delta)P_X+\delta P_Y)-P_X\propto P_Y-P_X$, therefore it is sensible to use \eqref{eq:J} as our training objective for is also sensible for \eqref{eqn:lfht mix def}, and we don't even need to observe the sample $Z$.

\subsection{Proof of \texorpdfstring{\cref{prop: consistent}}{Consistency of Estimating p-value}}
\begin{proof}
    In this proof we regard $\cal D \eqdef (X^\sf{tr},X^\sf{ev},Y^\sf{tr},Y^\sf{ev})$ and the parameters of the kernel $\omega$ as fixed. Recall that we are looking at the problem $\mLF$ with a misspecification parameter $R=0$ (see \Cref{thm: MMDUpper}). Given a test set $\{z_i\}_{i \in [m]}$, our test statistic is $T(\{z_i\}_{i\in[m]}) = \frac{1}{m}\sum_{i=1}^m f(z_i)$ where $$f(z_i) = \frac{1}{n_\sf{ev}}\sum_{j=1}^{n_\sf{ev}} \Big(K_\omega(z_i, Y_j^\sf{ev}) - K_\omega(z_i, X_j^\sf{ev})\Big).$$ In Phase 3 of \Cref{alg:learn_deep_kernel}, we observe the value $\widehat T = T(Z) = \frac1m\sum_{i=1}^m f(Z_i)$ and reject the null hypothesis for large values of $\widehat T$. Thus, the $p$-value is defined as 
    \begin{align*}
        p = p(Z,\cal D) \eqdef \P_{\widetilde Z \sim P_X^{\otimes m}}(T(\widetilde Z) > \widehat T). 
    \end{align*}
    Phase 2 of our \Cref{alg:learn_deep_kernel} produces random variables $T_1,\dots,T_k$ that all have the distribution of $T(\{\widetilde Z_i\}_{i \in [m]})$, so that $\one\{T_r \geq \widehat{T}\}$ ($r=1,\dots,k$) are unbiased estimates of the $p$-value. However, the $T_i$ are not independent, because they sample from the finite collection of calibration samples $X^\sf{cal}$. However, as $n_\sf{cal}\to\infty$ the covariances between $T_{r_1},T_{r_2}$ for $r_1\neq r_2$ tend to zero, and we obtain a consistent estimate of $p$. \end{proof}

    
    

\subsection{Proof of \texorpdfstring{\Cref{prop:exists_kernel}}{Equivalence}}
\begin{proof}
    
The test statistic $T(X,Y,Z)$ in \eqref{eq: teststats} is given by
\begin{equation*}
    T(X,Y,Z) = \frac1m \sum_{i=1}^m
 f_K(Z_i)
 \end{equation*}
where
\begin{equation*}
f_K(z) = \theta_{\widehat{P}_Y}(z) -\theta_{\widehat{P}_X}(z).
\end{equation*}
This simplifies to (consider $K(x, y)=f(x)f(y)$)
\begin{equation*}        
    f_K(z) = \left(\frac1n\sum_{j=1}^n  f(Y_j)   -\frac1n\sum_{j=1}^n  f(X_j)
       \right)f(z)=C(X, Y)f(z).
\end{equation*}where $C(X, Y)$ does not depend on $z$. Therefore, for any witness function $f$, we obtain the desired additive test. 
\end{proof}

\subsection{Additive Test Statistics}\label{sec: benchmarks}
In this section we prove accordingly that the test statistics of all of \textbf{MMD-M/G/O}, \textbf{SCHE}, \textbf{LBI}, \textbf{UME}, \textbf{RFM} are of the form $T_f(Z)=\frac1m\sum_{i=1}^mf(Z_i)$ (where $f$ might depends on $X,Y$). The test is to compare $T_f(Z)$ with some threshold $\gamma(X,Y)$.

Note that in the setting of \Cref{alg:learn_deep_kernel}, the $X$ and $Y$ here correspond to $X^{\sf{ev}}$ and $Y^{\sf{ev}}$.

\textbf{MMD-M/G/O}\quad As described in \eqref{eq: teststats} we have
\begin{align*}
    T_f(Z)=\frac1m\sum_{i=1}^m \left( \frac{1}{n} \sum_{j=1}^n \left( K(Z_i,Y_j)-K(Z_i,X_j) \right) \right).
\end{align*}

\textbf{SCHE}\quad As described in \Cref{sec: Classifier} we have
\begin{align*}
    T_f(Z)=\frac1m\sum_{i=1}^m \one\{\phi(Z_i)>t\}.
\end{align*}

\textbf{LBI}\quad As described in \Cref{sec: Classifier} we have
\begin{align*}
    T_f(Z)=\frac1m\sum_{i=1}^m \log\left(\frac{\phi(Z_i)}{1-\phi(Z_i)}\right).
\end{align*}

\textbf{UME}\quad As described in \cite{jitkrittum2018informative}, the UME statistic evaluates the squared witness function at ${J_q}$ test locations $W=\{w_k\}_{k=1}^{J_q}\subset\mc{X}$. Formally for any two distributions $P,Q$ we define
\begin{align*}
    U^2(P,Q) = \norm{\theta_Q-\theta_P}_{L^2(W)}^2 = \frac{1}{{J_q}} \sum_{k=1}^{J_q} (\theta_Q(w_k)-\theta_P(w_k))^2.
\end{align*}
However, we note a crucial difference that their result only considers the case of $n=m$, and their proposed estimator for $U^2(P_Z,P_X)$ can not be naturally extended to the case of $n\neq m$. Here we generalize it to $m\neq n$ where we (conveniently) use a biased estimate of their distance. 
 Given samples $X,Y,Z$ and a set of witness locations $W$, the test statistic is a (biased yet) consistent estimator of $U^2(P_Z, P_Y)-U^2(P_Z, P_X)$. Let $\psi_W(z)=\frac{1}{\sqrt{J_q}}(K(z,w_1),\dots,K(z,w_{J_q}))\in \mathbb{R}^{|W|}$ be the ``feature function,'' then:
\begin{align*}
    \widehat{U}^2(Z,X) &= 
    \left\|\frac{1}{m}\sum_{i=1}^{m}\psi_W\left(Z_i\right) -\frac{1}{n}\sum_{j=1}^{n}\psi_W\left(X_i\right)\right\|^2_2
    \\
    &=
    \left\|\frac{1}{m}\sum_{i=1}^m\psi_W\left(Z_i\right)\right\|_2^2
    +\left\|\frac{1}{n}\sum_{j=1}^n\psi_W\left(X_i\right)\right\|_2^2
    -\frac{2}{mn}\sum_{1\leq i\leq m, 1\leq j\leq n}\langle\psi_W\left(Z_i\right), \psi_W\left(X_j\right)\rangle
\end{align*}
Here $\braket{\cdot,\cdot}$ denotes the usual inner product.
Therefore, the difference between distances is
\begin{align*}
    \widehat{U}^2(Z,Y)-\widehat{U}^2(Z,X)
    =
    \frac1m\sum_{i=1}^m \left\langle  \psi_W\left(Z_i\right),
    \frac2n\sum_{j=1}^n (\psi_W\left(X_j\right)-\psi_W\left(Y_j) \right)  \right\rangle
    + F(X, Y)
\end{align*}
where $F$ is sum function based only on $X, Y$. This is clearly an additive statistic for $Z$.

\textbf{RFM}\quad Algorithm 1 in \cite{radhakrishnan2022feature} describes a method for learning a kernel from data given a binary classification task. For convenience lets concatenate the data to $X^{\text{RFM}}=(X,Y)\in\R^{2n\times d}$ and labels $y^{\text{RFM}}=(\vec{0}_n,\vec{1}_n)\in\R^{1\times 2n}$. Given a learned kernel $K$, we write the Gram matrix as $(K(X^{\text{RFM}},X^{\text{RFM}}))_{i,j}=K(X^{\text{RFM}}_i,X^{\text{RFM}}_j)$ ($1\leq i,j\leq 2n$). Let $K(X^{\text{RFM}},z)$ be a column vector with components $K(X^{\text{RFM}}_i,z)$ ($1\leq i\leq 2n$). The classifier is then defined as 
\begin{align}\label{eq:RFM classifier}
    f^{\text{RFM}}(z) = y^{\text{RFM}} \cdot K(X^{\text{RFM}},X^{\text{RFM}})^{-1} \cdot K(X^{\text{RFM}},z). 
\end{align}
Though in \cite{radhakrishnan2022feature} the kernel learned from RFM is used to construct a classifier as in \Cref{eq:RFM classifier}, since RFM is a feature learning method, we also apply the RFM kernel to our MMD test, namely 
\begin{align*}
    f^{\text{RFM to MMD}}(z) = \frac{1}{n} \sum_{j=1}^n \left( K(z,Y_j)-K(z,X_j) \right). 
\end{align*}


\section{Application: Diffusion Models vs CIFAR}\label{sec: CIFAR}
We defer a more fine-grained detail to our code submission, which includes executable programs (with PyTorch) once the  data-generating script from DDPM has been run (see README in the ./codes/CIFAR folder). 

\subsection{Dataset Details}
We use the CIFAR-10 dataset available online at \url{https://www.cs.toronto.edu/~kriz/cifar.html}, which contains 50000 colored images of size $32\times 32$ with 10 classes. For the diffusion generated images, we use the SOTA Hugging Face model (DDPM) that can be found at \url{https://huggingface.co/google/ddpm-CIFAR-10-32}. We generated 10000 artificial images for our experiments. The code can be found at our code supplements.

For dataset balancing, we randomly shuffled the CIFAR-10 dataset and used 10000 images as data in our code. Most of our experiments are conducted with the null $P_X$ as CIFAR images, and the alternate as $P_Y=\frac{2}{3}\cdot\text{CIFAR}+\frac{1}{3}\cdot\text{DDPM}$. To this end, we matched 20000 images from CIFAR to belong to the alternate hypothesis, and the remaining 30000 images to stay in the null hypothesis. For the alternate dataset, we simply sample without replacement from the $20000+10000$ mixture. This sampled distribution is \emph{almost} the same as mixing (so long as the sample bank is large enough compared to the acquired data, so that each item in the alternate has close to $1/3$ probability of being in DDPM, 
which is indeed the case). 
\begin{figure}
     \centering
     \includegraphics[scale=0.9]{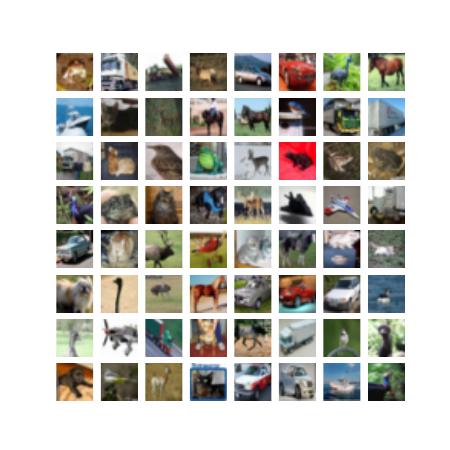}
     \includegraphics[scale=0.9]{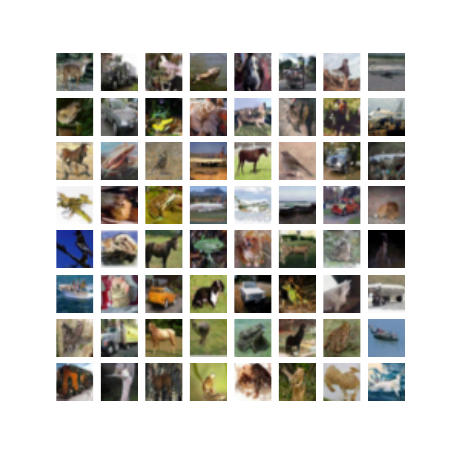}
     \caption{Data visualization for CIFAR-10 (left) vs DDPM diffusion generated images (right)} \end{figure}
 
\subsection{Experiment Setup and Benchmarks}
We use a standard deep Conv-net \cite{lopez2016revisiting}, which has been employed for SOTA GAN discriminator tasks in similar settings. It has four convolutional layers and one fully connected layer outputting the feature space of size $(300, 1)$. For SCHE and LBI, we simply added a linear layer of $(300, 2)$ after applying ReLU to the 300-dimensional layer and used the cross-entropy loss to train the network. Note that this is equivalent to first fixing the feature space and then performing logistic regression to the feature space. For kernels, we add extra trainable parameters after the $300$-d feature output.

For the $\mmd$-based  tests, we simply train the kernel on the neural net and evaluate our objective. For UME, we used a slightly generalized version of the original statistic in \cite{jitkrittum2018informative} which allows for comparison on randomly selected witness locations in the null hypothesis with $m\neq n$ (see \Cref{sec: benchmarks}). The kernel is trained using our heuristic (see \eqref{eq:J} and \Cref{sec:training obj justification}), with MMD replaced by UME. The formula for UME variance can be found in \cite{jitkrittum2018informative}.
For RFM, we use Algorithm 1 in \cite{radhakrishnan2022feature} to learn a kernel on (stochastic batched) samples, and then use our MMD test on the trained kernel.

We use 80 training epochs for most of our code from the CNN architecture (for classifiers, this is well after interpolating the training data and roughly when validation loss stops decreasing), and a batch size of 32 which has a slight empirical benefit compared to larger batch sizes. The learning rates are tuned separately in MMD methods for optimality, whereas for classifiers they follow the discriminator's original setting from \cite{lopez2016revisiting}. In Phase 2 of Algorithm 1, we choose $k=1000$ for the desired precision while not compromising runtime. For each task, we run $10$ independent models and report their performances as the mean and standard deviation of those $10$ runs as estimates. We refer to a full set of hyper-parameters in our code implementation. 

Our code is implemented in Python 3.7 (PyTorch 1.1) and was ran on an NVIDIA RTX 3080 GPU equipped with a standard torch library and dataset extensions. Our code setup for feature extraction is similar to that of \cite{liu2020learning}. For benchmark implementations, our code follows from the original code templated provided by the cited papers.

\begin{figure}
    \centering
   \includegraphics[scale=0.45]{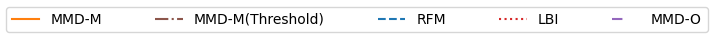}
    \includegraphics[scale=0.35]{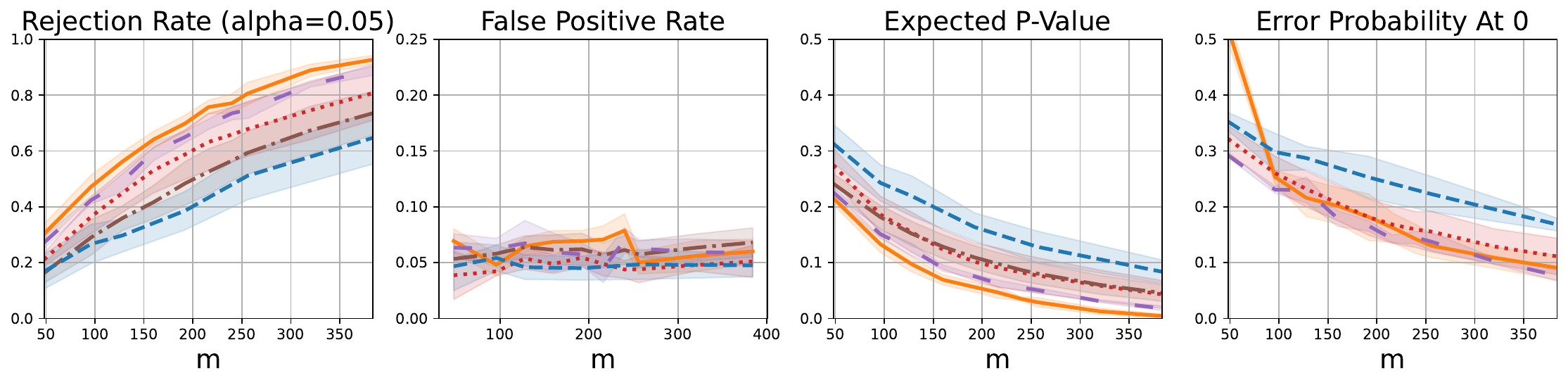}
    \caption{Relevant plots following the setting in \Cref{fig:cifar main} (in the main text) of fixing $n_{\sf{tr}}=1920$ and varying sample size $m$ in the x-axis for the comparison with missing benchmarks. Errorbars are projected showing standard deviation across 10 runs. We replaced part (d) in \Cref{fig:cifar main} (in the main text) to a sanity check in our FPR when thresholded at $\alpha=0.05$.} \label{fig: cifar_sup}
\end{figure}

\subsection{Sample Allocation}\label{appendix:sample allocation}
We make a comment on why \eqref{eq: test} is \emph{different} from just thresholding $\widehat {\mmd^2}(Z,Y^{\sf{tr}})-\widehat{\mmd^2}(Z,X^{\sf{tr}})$ at 0, which was what we did in part (c) of \Cref{fig:cifar main} (and hence the difference along the curve of MMD-M vs Figure \ref{fig:trade off}). Our theory assumes that the samples are i.i.d. conditioned on the kernel being chosen already. However, in the experiments, the kernel is dependent on the training data. Therefore, to evaluate the MMD estimate (between experimentations), one needs extra data that does not intersect with training.

In fact, it can be experimentally shown by comparing Figure \ref{fig:trade off} and Figure 2(c) that doing so (while reducing the sample complexity on 
$n_{\sf{ev}}$) hurts performance. Indeed, we found out that when $X^{\sf{ev}},Y^{\sf{ev}}$ are non-intersecting with training, performance is (almost) always better at a cost of hurting the overall sample complexity of $n$.
\subsection{Remarks on Results}
\Cref{fig: cifar_sup} lists all of our benchmarks in the setting of \Cref{fig:cifar main} (in the main text) on missing benchmarks, where the last figure is replaced by the false positive rate at thresholding at $\alpha=0.05$ to verify our results. As mentioned in the main text, our MMD-M method consistently outperforms other benchmarks on both the expected $p$-value (of alternate) and rejection rate at $\alpha=0.05$, while all of our tests observe a empirical false positive rate close to $\alpha=0.05\%$ (Part (b)), showing the consistency of methods.

\section{Application: Higgs-Boson Detection}
\label{appendix:higggs}
\subsection{Dataset Details}
We use the Higgs dataset available online at \url{http://archive.ics.uci.edu/ml/datasets/HIGGS}, produced using Monte Carlo simulations \cite{baldi2014searching}. The dataset is nearly balanced, containing $5,829,122$ signal instances and $5,170,877$ background instances. Each instance is a $28$-dimensional vector, consisting of $28$ features. The first $21$ features are  kinematic properties measured by the detectors in the accelerator, such as momentum and energy. The last $7$ properties are \emph{invariant masses}, derived from the first $21$ features.

\begin{figure}[ht]
    \centering
    \includegraphics[width=.7\textwidth]{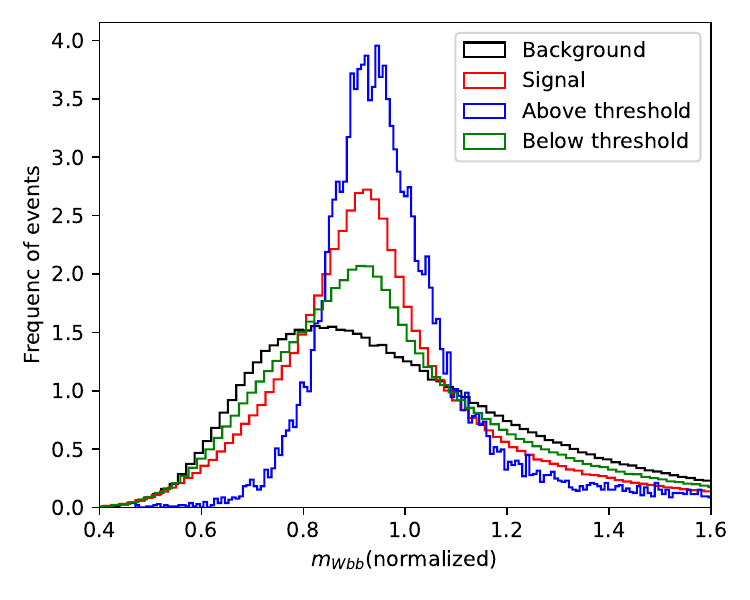}
    \caption{This figure visualizes the distribution of the $26$th feature, the invariant mass $m_{Wbb}$. The red and black lines are the histograms of the original dataset. We employ MMD-M as a classifier, trained and evaluated using $n_{\sf{tr}}=1.3\times 10^6$ and $n_{\sf{ev}}=n_{\sf{opt}}=2\times 10^4$ through \Cref{algo:higgs_thres}. The blue(green) line represents all instances $z$'s whose ``witness scores'' $f(z;X^{\sf{ev}},Y^{\sf{ev}})$'s are larger(smaller) than $t_{\sf{opt}}$.}
    \label{fig:invariant_mass}
\end{figure}

\subsection{Experiment Setup and Training Models}

The modified \Cref{alg:learn_deep_kernel} is shown in \Cref{algo:higgs_orig} and \Cref{algo:higgs_thres}. Compared with \Cref{algo:higgs_orig}, we implement the thresholding trick (\Cref{sec:threshold trick}) in \Cref{algo:higgs_thres}.

\subsubsection{Configuration and Model Architecture}
We implement all methods in Python 3.9 and PyTorch 1.13 and run them on an NVIDIA Quadro RTX 8000 GPU. 

For all classifier-based methods in this study ({SCHE} and {LBI}), we adopt the same architecture as previously proposed in \cite{baldi2014searching}. The classifiers are six-layer neural networks with 300 hidden units in each layer, all employing the tanh activation function. For {SCHE}, the output layer is a single sigmoid unit and we utilize the binary cross-entropy loss for training. For {LBI}, the output layer is a linear unit and we utilize the binary cross entropy loss combined with a logit function (which is more numerically stable than simply using a sigmoid layer followed by a cross entropy loss).

For all MMD-based methods ({MMD-M}, {MMD-G}, {MMD-O}, and {UME}), the networks $\varphi$ and $\varphi'$ are both six-layer neural networks with 300 ReLU units in each layer. The feature space, which is the output of the neural network $\varphi$, is set to be $100$-dimensional. Here {UME} has the same kernel architecture as {MMD-M}, and the number of test locations is set to be $J_q=4096$. For {RFM}, we adopt the same architecture as in \cite{radhakrishnan2022feature}, where the kernel is $K_M(x,y)=\exp(-\gamma(x-y)^TM(x-y))$ with a constant $\gamma$ and a learnable positive semi-definite matrix $M$. We set $\gamma\equiv1$.

The neural networks are initialized using the default setting in PyTorch, and the bandwidths $\sigma,\sigma'$ are initialized using the \emph{median heuristic} \citep{gretton2012optimal}. The parameter $\tau$ is initially set to $0.5$. For UME, the witness locations $W$ are initially randomly sampled from the training set. For RFM, the initial $M$ equals the median bandwidth times an identity matrix.

\subsubsection{Training}

The size of our training set, denoted as $n_{\sf{tr}}$, varies from $1.0\times 10^2$ to $1.6\times 10^6$. For a given $n_{\sf{tr}}$, we select the first $n_{\sf{tr}}$ datapoints from each class of the Higgs dataset to form $X^\sf{tr}$ and $Y^\sf{tr}$, i.e., $|X^\sf{tr}|=|Y^\sf{tr}|=n_{\sf{tr}}$. Subsequently, we randomly select $n_{\sf{validation}}=\min(\sqrt{10 n_{\sf{tr}}},0.1 n_{\sf{tr}})$ points from each of $X_{\sf{tr}},Y_{\sf{tr}}$ to constitute the validation set, while the remainder of $X_{\sf{tr}},Y_{\sf{tr}}$ are used for running gradient descent. The optimizer is set to be a minibatch SGD, with a batch size of $1024$, a learning rate of $0.001$, and a momentum of $0.99$. Training is halted once the validation loss stops to decrease for $10$ epochs, then we choose the checkpoint (saved for each epoch) with the smallest validation loss thus far as our trained model. Beyond the general setting above, in RFM a batch size of $1024$ doesn't work well and instead we use a batch size of $20,000$.

\subsection{Evaluating the Performance}

\subsubsection{Evaluating the p-Value with the Methodology of \Cref{alg:learn_deep_kernel}}
We call the ``witness score'' of an instance $z\in\mc{X}$ as
\begin{equation}
\label{eq: appendix_score}
    f(z;X^\sf{ev},Y^\sf{ev})= \frac{1}{n_\sf{cal}} \sum_{i=1}^{n_{\sf{cal}}} \left( k(z,Y_i^\sf{ev})-k(z,X_i^\sf{ev})\right).
\end{equation}
For a vector of instances $Z=(Z_1,\ldots,Z_m)$, we write $$f(Z;X^\sf{ev},Y^\sf{ev})=(f(Z_1;X^\sf{ev},Y^\sf{ev}),\ldots,f(Z_m;X^\sf{ev},Y^\sf{ev})).$$
The testing procedure is summarized in Phases 2, 3 and 4 in  \Cref{algo:higgs_orig} and  \Cref{algo:higgs_thres}. In the Higgs experiment, we utilize the Gaussian approximation method to determine the p-values when the witness function $f$ is not thresholded, which allows us to reach very small p-values and errors under limited computational resource. In cases where the score function $f$ is thresholded by a value $t$, using the Binomial distribution as in \Cref{algo:higgs_thres} is more precise and also fast enough. 

Given a trained kernel $K$ trained on $X^{\sf{tr}}$ and $Y^{\sf{tr}}$, we set $X^\sf{ev}=X^\sf{tr}$ and $Y^\sf{ev}=Y^\sf{tr}$, and accordingly $n_{\sf{ev}}=n_{\sf{tr}}$. This results in a more efficient use of data (since we reuse $X^{\sf{tr}},Y^{\sf{tr}}$ also as $X^{\sf{ev}},Y^{\sf{ev}}$). Then, out of the untouched portion of the data, we randomly choose $n_\sf{cal}=20,000$ datapoints from both classes to populate $X^\sf{cal}$ and $Y^\sf{cal}$, i.e., $|X^\sf{cal}|=|Y^\sf{cal}|=n_\sf{cal}=20,000$.
In addition to the general setting above, for RFM, we need to solve a $2n_{\sf{ev}}$-dimensional linear equation during inference, which arises from the inverse matrix in \Cref{eq:RFM classifier} (solving $K(X^{\text{RFM}},X^{\text{RFM}}) \boldsymbol{u} =(y^{\text{RFM}})^T$ for $\boldsymbol{u}\in\R^{2n_{\sf{ev}}}$). So we set $n_{\sf{ev}}=\min(n_{\sf{tr}},10,000)$ that $X_{\sf{ev}},Y_{\sf{ev}}$ are randomly sampled from the training set.

In order to compare different benchmarks, we evaluate the expected significance of discovery on a mixture of 1000 backgrounds and 100 signals. For each benchmark and each $n_{\sf{tr}}$, we train $10$ independent models. Then for each trained model we proceed through the Phases 2, 3 (and 4) in \Cref{algo:higgs_orig} and \Cref{algo:higgs_thres} by $10$ times for $10$ different $(X^{\sf{ev}},X^{\sf{cal}},X^{\sf{opt}},Y^{\sf{ev}},Y^{\sf{cal}},Y^{\sf{opt}})$.  The mean and standard deviation from these $100$ runs are reported in \Cref{fig:full_higgs_discover}.

We also display in \Cref{fig: higgs_tradeoff_sup} the trade-off b $(m,n_{\sf{ev}})$ and $(m, n_{\sf{tr}})$ to reach certain levels of significance of discovery in MMD-M. From the bottom left plot, we see that the (averaged) significance is not sensitive to $n_{\sf{ev}}$ when $\lg n_{\sf{ev}}$ is large. So taking $n_{\sf{ev}}=20,000$ is sufficient.  

\subsubsection{Evaluating the Error of the Test \eqref{eq: test}}

We set the parameters to be $\delta=0.1$ and $\pi=\frac12\delta$ in our experiments. As explained \Cref{appendix:sample allocation}, here we no longer take $X^{\sf{ev}}=X^{\sf{tr}}$. Empirically, taking $X^{\sf{ev}}=X^{\sf{tr}}$ yields a very bad threshold $\gamma(X^{\sf{ev}},Y^{\sf{ev}},\pi)$.\footnote{If the kernel $K(\cdot,\cdot)=K_{X^{\sf{tr}},Y^{\sf{tr}}}(\cdot,\cdot)$ is independent of $X^{\sf{ev}},Y^{\sf{ev}}$, then we have $\gamma(X^{\sf{ev}},Y^{\sf{ev}},\delta/2)\approx\frac12\left(\E_{Z\sim P_x}[T(X^{\sf{ev}},Y^{\sf{ev}},Z)]+\E_{Z\sim\delta P_Y+(1-\delta)P_X}[T(X^{\sf{ev}},Y^{\sf{ev}},Z)]\right)$. However this is no longer true if $(X^{\sf{tr}},Y^{\sf{tr}})$ and $(X^{\sf{ev}},Y^{\sf{ev}})$ intersect.} 
Instead, 
$X^{\sf{ev}}$ is sampled from untouched datapoints other than $X^{\sf{tr}}$, and the same applies for $Y$. We still take $n_{\sf{ev}}=n_{\sf{tr}}$ here, resulting in a total size of $n_{\sf{ev}}+n_{\sf{tr}}=2n_{\sf{tr}}$. Specifically, when $n_{\sf{ev}}\geq 10,000$, computing a $n_{\sf{ev}}\times n_{\sf{ev}}$ Gram matrix becomes computationally expensive, so we adopt Monte Carlo method to compute $\gamma(X^{\sf{ev}},Y^{\sf{ev}},\pi)$, in which we subsample $10,000$ points from $X^{\sf{ev}}$ and $Y^{\sf{ev}}$ to calculate $\gamma$ and repeat this process 100 times.

Again, we utilize the Gaussian approximation. Recall that the test is to compare $T=\frac1m\sum_{i=1}^m f(Z_i)$ with $\gamma$. The type 1 and type 2 error are estimated as $\text{CDF}_{\mc{N}(0,1)} \left( -\frac{ \gamma(X^{\sf{ev}},Y^{\sf{ev}},\pi)-\E[f|H_0] }{\sqrt{\var(f|H_0)/m}} \right)$ and $\text{CDF}_{\mc{N}(0,1)} \left(-\frac{\E[f|H_1]-\gamma(X^{\sf{ev}},Y^{\sf{ev}},\pi)}{\sqrt{\var(f|H_1)/m}}\right)$ for the witness function $f$, which can be estimated efficiently using the calibration samples $X^\sf{cal},Y^\sf{cal}$. 

We consider both the regimes of fixing kernels and varying kernels (training kernel based on $n$).  The results are shown in the top plot in Figure \ref{fig:trade off} and the top plot in \Cref{fig: higgs_tradeoff_sup}. For each point on the plot, we train 30 independent models and test each model 10 times, and report the average of these 300 runs. In both plots, we observe the asymmetric $m$ vs $n$ trade-off.
\\


\begin{breakablealgorithm}
    \caption{Estimate the significance of discovery of an input $Z_{\text{test}}$, using the original statistic}
    \begin{algorithmic}
    \label{algo:higgs_orig}
        \STATE \textbf{Input: } $(X^\sf{tr}$, $X^\sf{ev}, X^\sf{cal})$, $(Y^\sf{tr}, Y^\sf{ev}, Y^\sf{cal})$; parametrized kernel $K_\omega$; input $Z_{\text{test}}$.
        \STATE \textit{\# Phase 1: Kernel training on $X^\sf{tr}$ and $Y^\sf{tr}$\hfill}
            \STATE $\omega \gets \arg\max_\omega^{\text{optimizer}}\hat J(X^{\sf{tr}}, Y^{\sf{tr}}; K_w)$ \hfill \textit{\# maximize objective $\widehat J(X^\sf{tr},Y^\sf{tr};K_\omega)$ as in \eqref{eq:J} }
        \vspace{1mm}
        
        \STATE \textit{\# Phase 2: Distributional calibration of test statistic}
            \STATE  $\text{Scores}^{(0)} \gets f(X^\sf{cal};X^\sf{ev},Y^\sf{ev}) $
                        \hfill \textit{\# $\text{Scores}^{(0)}$ has a length of $n_{\sf{cal}}$}
            \STATE  $\text{Scores}^{(1)} \gets f(Y^\sf{cal};X^\sf{ev},Y^\sf{ev})$ 
                        \hfill \textit{\# $\text{Scores}^{(1)}$ has a length of $n_{\sf{cal}}$}
            \STATE $\theta_0 \gets \text{mean}(\text{Scores}^{(0)})$ 
                        \hfill \textit{\# estimate $\E[f(Z)|Z\sim P_X]$}
            \STATE $\theta_1 \gets \text{mean}(\text{Scores}^{(1)})$ 
                        \hfill \textit{\# estimate $\E[f(Z)|Z\sim P_Y]$}
            \STATE $\sigma_0 \gets \text{std}(\text{Scores}^{(0)})$ 
                        \hfill \textit{\# estimate $\sqrt{\var[f(Z)|Z\sim P_X]}$ }
        \vspace{1mm}
        
        \STATE \textit{\# Phase 3: Inference with input $Z_{\text{test}}$}
            \STATE $m\, \gets \mathrm{length}(Z_{\text{test}})$
            \STATE $\,T\, \gets T_f(Z_{\text{test}};X^\sf{ev},Y^\sf{ev})=\text{mean}( f(Z_{\text{test}}; X^\sf{ev},Y^\sf{ev} ))$
                        \hfill \textit{\# compute test statistic}
            \STATE $ Z_{\text{discovery}} \gets \frac{T-\theta_0}{\sigma_0/\sqrt{m}}$ 
        \STATE \textbf{Output: } Estimated significance: $Z_{\text{discovery}}$ 
    \end{algorithmic}
\end{breakablealgorithm}

\begin{breakablealgorithm}
    \caption{Estimate the significance of discovery of an input $Z_{\text{test}}$, applying the thresholding trick}
    \begin{algorithmic}
    \label{algo:higgs_thres}
        \STATE \textbf{Input: } $(X^\sf{tr}$, $X^\sf{ev}, X^\sf{cal}, X^\sf{opt})$, $(Y^\sf{tr}, Y^\sf{ev}, Y^\sf{cal}, Y^\sf{opt}$); parametrized kernel $K_\omega$; input $Z_{\text{test}}$.
        \STATE \textit{\# Phase 1: Kernel training on $X^\sf{tr}$ and $Y^\sf{tr}$\hfill}
            \STATE $\omega \gets \arg\max_\omega^{\text{optimizer}}\hat J(X^{\sf{tr}}, Y^{\sf{tr}}; K_w)$ \hfill \textit{\# maximize objective $\widehat J(X^\sf{tr},Y^\sf{tr};K_\omega)$ as in \eqref{eq:J} }
        \vspace{1mm}
        
        \STATE \textit{\# Phase 2: Find the best threshold}
            \STATE $ \text{Scores}^{(0)} \gets f(X^\sf{opt};X^\sf{ev},Y^\sf{ev})$
            \STATE $ \text{Scores}^{(1)} \gets f(Y^\sf{opt};X^\sf{ev},Y^\sf{ev})\hfill \textit{\# witness function as in \eqref{eq: appendix_score}}$
            \FOR{$i = 1,2,...,2n_{\sf{opt}}$}
                \STATE $t = (\text{Scores}^{(0)}\cup\text{Scores}^{(1)})[i]$
                    \STATE $\text{TP}, \text{TN} = \text{mean}(\text{Scores}^{(1)} > t),  \text{mean}(\text{Scores}^{(0)} < t)$ \hfill \textit{\# true positive and true negative rate}
                \STATE $\text{power}_i=\frac{\text{TP}+\text{TN}-1}{\sqrt{\text{TN}(1-\text{TN})}}$\hfill \textit{\# find $t$ to maximize the (estimated) p-value}
            \ENDFOR
            \STATE $t_{\sf{opt}} = (\text{Scores}^{(0)}\cup\text{Scores}^{(1)})[\argmax_i \text{power}_i$]
        \vspace{1mm}
        
        \STATE \textit{\# Phase 3: Distributional calibration of test statistic (under null hypothesis)}
            \STATE  $\text{Scores}^{(0)} \gets (f(X^\sf{cal};X^\sf{ev},Y^\sf{ev})>t)$ 
                        \hfill \textit{\# $\text{Scores}^{(0)}\in\{0,1\}^{n_{\sf{ev}}}$}
            \STATE  $\text{Scores}^{(1)} \gets (f(Y^\sf{cal};X^\sf{ev},Y^\sf{ev})>t)$ 
                        \hfill \textit{\# $\text{Scores}^{(1)}\in\{0,1\}^{n_{\sf{ev}}}$}
            \STATE $\theta_0 \gets \text{mean}(\text{Scores}^{(0)})$ 
                        \hfill \textit{\# estimate $\E[f_t(Z)|Z\sim P_X]\in[0,1]$}
            \STATE $\theta_1 \gets \text{mean}(\text{Scores}^{(1)})$ 
                        \hfill \textit{\# estimate $\E[f_t(Z)|Z\sim P_Y]\in[0,1]$}
        \vspace{1mm}
        \STATE \textit{\# Phase 4: Inference with input $Z_{\text{test}}$}
        \vspace{1mm}
            \STATE $m\, \gets \mathrm{length}(Z_{\text{test}})$
            \STATE $\,T\, \gets T_f(Z_{\text{test}};X^\sf{ev},Y^\sf{ev})=\text{mean}( f(Z_{\text{test}}; X^\sf{ev},Y^\sf{ev})>t)$
                        \hfill \textit{\# compute test statistic}
            \STATE $Z_{\text{discovery}} \gets \textup{CDF}^{-1}_{\mc{N}(0,1)}(\textup{CDF}_{\textup{Bin}(m,\theta_0)}(T))$ 
        \STATE \textbf{Output: } Estimated significance: $Z_{\text{discovery}}$ 
    \end{algorithmic}
\end{breakablealgorithm}



\begin{figure}[H]
    \centering
    \includegraphics[width=0.7\textwidth]{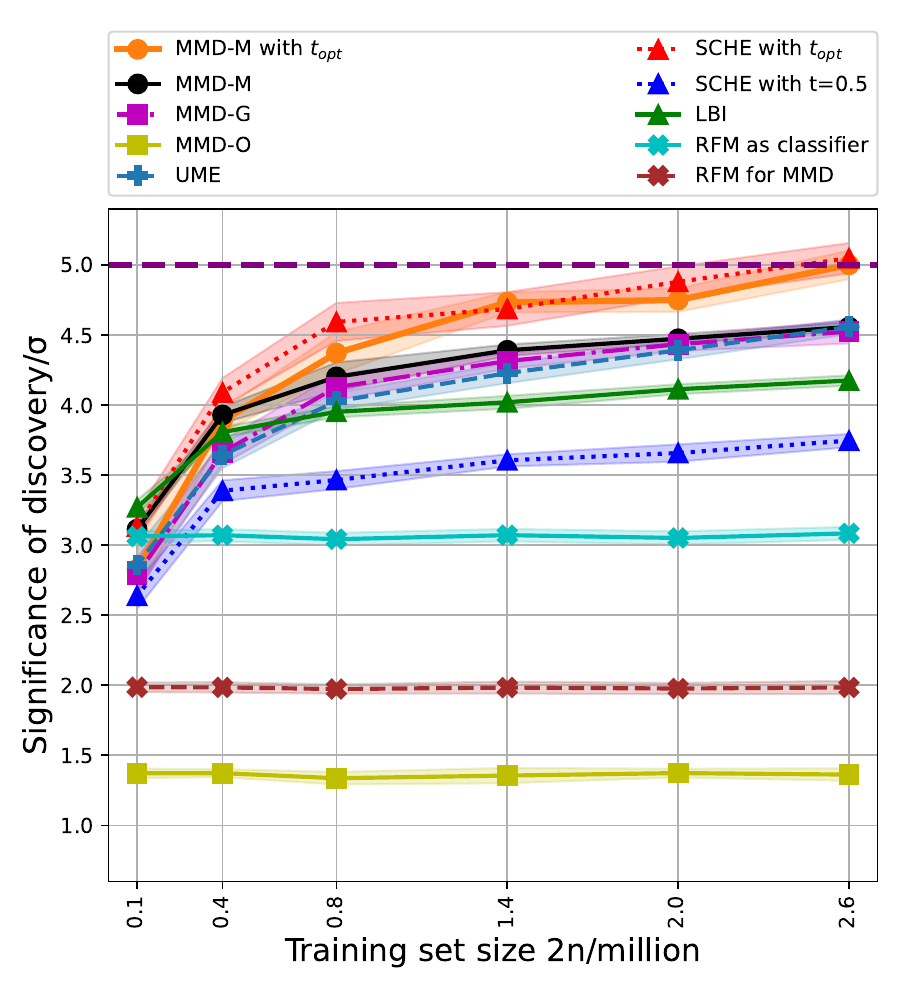}
    \caption{Complete image of Figure \ref{fig:trade off} in the main text. The mean and standard deviation are calculated based on $100$ runs. See \Cref{appendix:higggs} for details.}
    \label{fig:full_higgs_discover}
\end{figure}

\begin{figure}[H]
    \centering
    \includegraphics[width=0.65\textwidth]{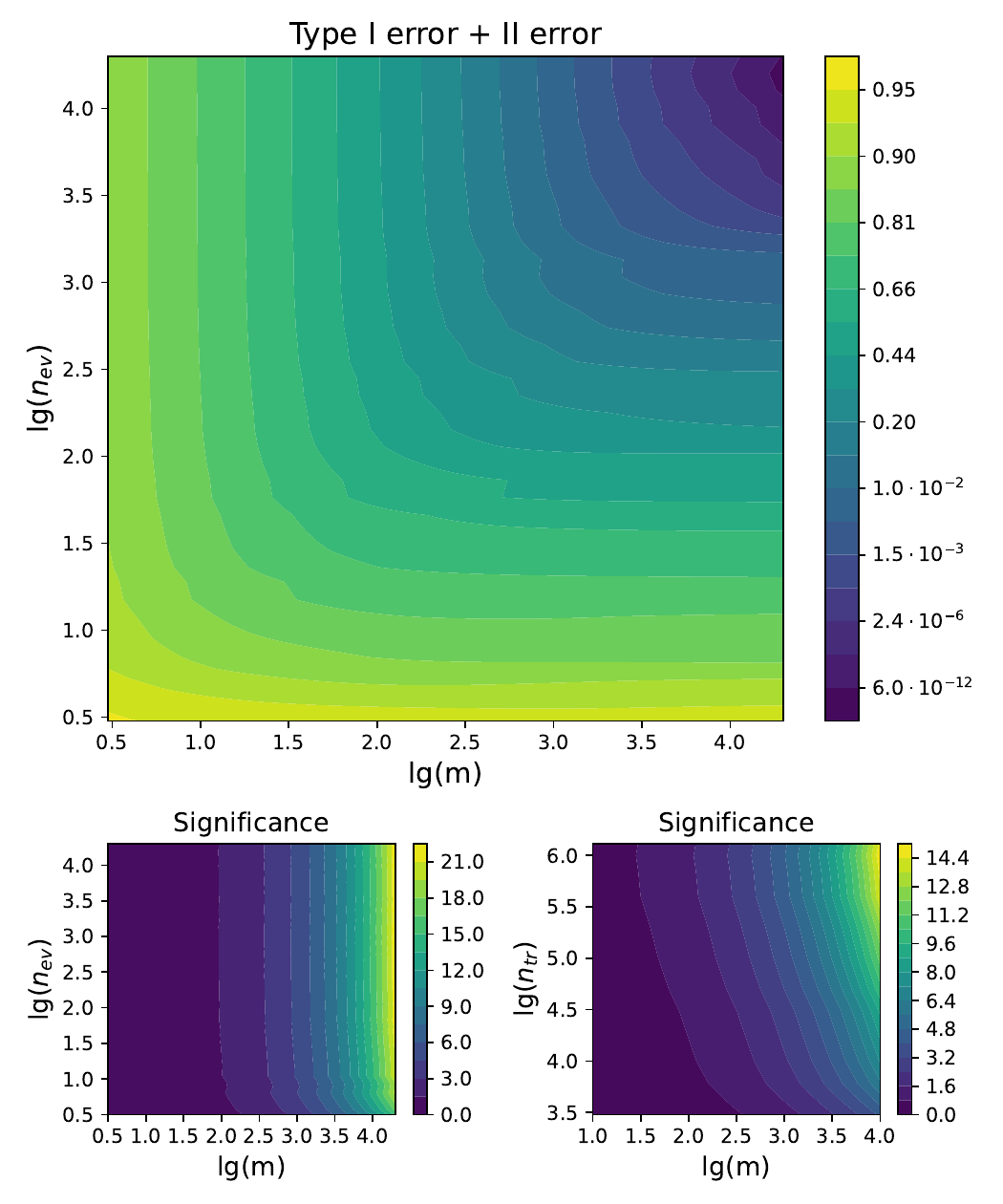}
    \caption{The top plot displays the 
    $(m,n_{\sf{ev}})$ trade-off to reach certain levels of total error  using $n_{\sf{tr}}=1.3\times 10^6$ in MMD-M. The bottom figures show the trade-off of $(m,n_{\sf{ev}})$ and $(m, n_{\sf{tr}})$ to reach certain level of significance of discovery in MMD-M. In the bottom left figure, we fix $n_{\sf{tr}}=1.3\times 10^6$. In the bottom right figure, we fix $n_{\sf{ev}}=20,000$. See \Cref{appendix:higggs} for details.}\label{fig: higgs_tradeoff_sup}
\end{figure}

\section{Limitations and Future Directions}
Finally, we discuss several limitations of our work and raise open questions that we hope will be addressed in future works. From the theoretical side of our arguments, we point out several aspects. First, our upper bound (on the minimax sample complexity) \Cref{thm: MMDUpper} has a likely sub-optimal dependence on $\alpha, \delta$. Second, it might be possible to improve our lower bound to a more natural form by replacing $\|\lambda\|_{2,J_{\epsilon}^\star}$ by $\|\lambda\|_2$ and removing the constraint that the top eigenfunction has to be constant. Third, it remains open to extend our theory to include data-dependent $K$, as opposed to fixed $K$. 

Empirically, our proposal \Cref{alg:learn_deep_kernel} can be inefficient in Phase 2 (prior works such as \cite{liu2020learning} have used permutation-based arguments for a more efficient estimate), which we adopted due to its simplicity and universality in all benchmarks. Moreover, one might hope that $\LF$/$\mLF$ can be extended to more complex applications, such as text data or videos. Such questions are important to investigate as a future direction.

\end{document}